%% file: main.tex
\documentclass{article}
\usepackage{iclr2023_conference,times}

\input{headers/header-lqa}

\input{headers/header}

\usepackage{natbib}
\usepackage{hyperref}
\usepackage{adjustbox}
\input{headers/lqa/references}

\newcommand\blfootnote[1]{%
	\begin{NoHyper}%
	\begingroup
	\renewcommand\thefootnote{}\footnote{#1}%
	\addtocounter{footnote}{-1}%
	\endgroup
	\end{NoHyper}%
}

\graphicspath{ {./images/} }

\title{Human-Guided Fair Classification for \protect \\Natural Language Processing}

\author{Florian E.~Dorner\textsuperscript{\rm 1,2},
Momchil Peychev\textsuperscript{\rm 1},
Nikola Konstantinov\textsuperscript{\rm 1},
Naman Goel\textsuperscript{\rm 3}, \\
\textbf{Elliott Ash\textsuperscript{\rm 1},
Martin Vechev\textsuperscript{\rm 1}} \\
\textsuperscript{\rm 1}ETH Zurich,
\textsuperscript{\rm 2}MPI for Intelligent Systems, Tübingen,
\textsuperscript{\rm 3}University of Oxford\\
Correspondence to: \texttt{florian.dorner@tuebingen.mpg.de}
}

\iclrfinalcopy 

\begin{document}

\maketitle


\begin{abstract}
	\input{abstract}
\end{abstract}

\input{introduction}

\input{related}
\input{method}
\input{experiments}
\input{conclusion}
\input{ethics}
\input{reproducibility}

\input{acknowledgements}

\message{^^JLASTBODYPAGE \thepage^^J}

\clearpage
\bibliography{references}
\bibliographystyle{iclr2023_conference}

\message{^^JLASTREFERENCESPAGE \thepage^^J}


\ifbool{includeappendix}{%
	\clearpage
	\appendix
	\input{appendix}

}{}

\message{^^JLASTPAGE \thepage^^J}

\end{document}

%% file: headers/header-lqa.tex
\usepackage[utf8]{inputenc} 
\usepackage{lipsum} 

\usepackage[T1]{fontenc}

\usepackage{etoolbox}
\newbool{includeappendix}
\setbool{includeappendix}{true}

\input{headers/lqa/overfull}

\input{headers/lqa/comments}

\input{headers/lqa/abbreviations}

\input{headers/lqa/listing}

\usepackage{amsthm}

%% file: headers/lqa/overfull.tex
\ifdefined\isoverfull
\else
\fi

%% file: headers/lqa/comments.tex
\newcommand{\mycomment}[1]{}


%% file: headers/lqa/abbreviations.tex
\newcommand{\eg}{e.g., }

%% file: headers/lqa/listing.tex
\usepackage{listings}

\usepackage{textcomp}

\usepackage{xcolor}

\usepackage[scaled=0.8]{beramono}

\definecolor{ckeyword}{HTML}{7F0055}
\definecolor{ccomment}{HTML}{3F7F5F}
\definecolor{cstring}{HTML}{2A0099}

\lstdefinestyle{numbers}{
	numbers=left,
	framexleftmargin=20pt,
	numberstyle=\tiny,
	firstnumber=auto,
	numbersep=1em,
	xleftmargin=2em
}

\lstdefinestyle{layout}{
	frame=none,
	captionpos=b,
}

\lstdefinestyle{comment-style}{
	morecomment=[l]//,
	morecomment=[s]{/*}{*/},
	commentstyle={\color{ccomment}\itshape},
}

\lstdefinestyle{string-style}{
	morestring=[b]",%
	morestring=[b]',%
	stringstyle={\color{cstring}},
	showstringspaces=false,%
}

\lstdefinestyle{keyword-style}{
	keywordstyle={\ttfamily\bfseries},
	morekeywords={
		function,
		constructor,
		int,
		bool,
		return,
		returns,
		uint
	},
	morekeywords = [2]{},
	keywordstyle = [2]{\text},
	sensitive=true,
}

\lstdefinestyle{input-encoding}{
	inputencoding=utf8,
	extendedchars=true,
	literate=
	{ℝ}{$\reals$}1%
	{→}{$\rightarrow$}1%
	{α}{$\alpha$}1%
	{β}{$\beta$}1%
	{λ}{$\lambda$}1%
	{θ}{$\theta$}1%
	{ϕ}{$\phi$}1%
}

\lstdefinestyle{escaping}{
	moredelim={**[is][\color{blue}]{\%}{\%}},
	escapechar=|,
	mathescape=true
}

\lstdefinestyle{default-style}{
	basicstyle=\fontencoding{T1}\ttfamily\footnotesize,
	style=numbers,
	style=layout,
	style=comment-style,
	style=string-style,
	style=keyword-style,
	style=input-encoding,
	style=escaping,
	tabsize=2,
	upquote=true
}

\lstdefinelanguage{BASIC}{
	language=C++,
	style=default-style
}[keywords,comments,strings]%

%% file: headers/header.tex
\usepackage{booktabs}       
\usepackage{caption}
\usepackage{enumitem}
\usepackage{graphicx}
\usepackage{makecell}
\usepackage{microtype}      
\usepackage{multirow}
\usepackage{nicefrac}       
\usepackage{siunitx}
\usepackage[caption=false]{subfig}
\usepackage[normalem]{ulem}
\usepackage{url}            

\usepackage{amsfonts,amssymb}       
\input{math_commands.tex}

\usepackage{bbm}


%% file: math_commands.tex
\usepackage{amsmath,amsfonts,bm}

\def\1{\bm{1}}

\DeclareMathAlphabet{\mathsfit}{\encodingdefault}{\sfdefault}{m}{sl}
\SetMathAlphabet{\mathsfit}{bold}{\encodingdefault}{\sfdefault}{bx}{n}



%% file: headers/lqa/references.tex
\usepackage[capitalize]{cleveref}

\crefrangeformat{section}{\S#3#1#4\crefrangeconjunction\S#5#2#6}

\crefmultiformat{section}{\S#2#1#3}{\crefpairconjunction\S#2#1#3}{\crefmiddleconjunction\S#2#1#3}{\creflastconjunction\S#2#1#3}

\newcommand{\crefrangeconjunction}{--}

\crefname{section}{Sec.}{Sections}

\crefname{listing}{Lst.}{listings}
\crefname{line}{Lin.}{Lin.}
\crefname{appendix}{App.}{App.}

\newcommand{\app}[1]{%
	\ifbool{includeappendix}{\cref{#1}}{the appendix}%
}
\newcommand{\App}[1]{%
	\ifbool{includeappendix}{\cref{#1}}{The appendix}%
}

%% file: abstract.tex
%

Text classifiers have promising applications in high-stake tasks such as resume screening and content moderation.  These classifiers must be fair and avoid discriminatory decisions by being invariant to perturbations of sensitive attributes such as gender or ethnicity. However, there is a gap between human intuition about these perturbations and the formal similarity specifications capturing them. While existing research has started to address this gap, current methods are based on hardcoded word replacements, resulting in specifications with limited expressivity or ones that fail to fully align with human intuition (e.g., in cases of asymmetric counterfactuals). This work proposes novel methods for bridging this gap by discovering expressive and intuitive individual fairness specifications. We show how to leverage unsupervised style transfer and GPT-3's zero-shot capabilities to automatically generate expressive candidate pairs of semantically similar sentences that differ along sensitive attributes. We then validate the generated pairs via an extensive crowdsourcing study, which confirms that a lot of these pairs align with human intuition about fairness in the context of toxicity classification. Finally, we show how limited amounts of human feedback can be leveraged to learn a similarity specification that can be used to train downstream fairness-aware models. 


%% file: introduction.tex
\section{Introduction}

With the rise of pretrained 
large language models~\citep{sun2019fine}, text classifiers can now be
employed in tasks related to automated hiring~\citep{bhatia2019end}, content
moderation~\citep{rieder2021fabrics} and social science research~\citep{widmer2022media}.
They are also part of machine learning pipelines for unsupervised style
transfer~\citep{reid2021lewis}
or reducing the toxicity of language model outputs~\citep{welbl2021challenges}.
However, text classifiers have been shown to often exhibit bias based on sensitive attributes such as gender~\citep{de2019bias} or demographics~\citep{garg2019counterfactual},
even for tasks in which these dimensions should be irrelevant. This can lead to unfair
and discriminatory decisions, distort analyses based on these classifiers, or
propagate undesirable demographic stereotypes to downstream applications.
The intuition that certain demographic indicators should not influence decisions
can be formalized in terms of the concept of
\textit{individual fairness}~\citep{dwork2012fairness},
which posits that \textit{similar inputs}
should be \textit{treated similarly} by machine learning
systems. While in a classification setting similar treatment for two inputs
can naturally be defined in terms of both inputs being labeled the same, the notion of input similarity should capture the intuition that certain input characteristics should not influence model decisions.

\paragraph{Key challenge: generating valid, intuitive and diverse fairness constraints}

A key challenge when applying the individual fairness framework is defining the similarity notion $\phi$. Indeed, the definition is often contentious, as fairness is a subjective concept: what counts as a valid demographic indicator,
as opposed to a problematic stereotype? Counterfactual definitions of
similarity~\citep{kusner2017counterfactual} offer a principled solution, but they shift the burden towards the underlying causal
model, whose definition can often be similarly contentious. While many other definitions have been proposed, it is widely recognized that the similarity of inputs
can often be highly task dependent ~\citep{dwork2012fairness, barocas2019fairness}, \eg two biographies that are identical except
for indicators of gender may be considered similar in a professional context, but not
in the context of online dating.

In the context of text classification, most existing works have cast similarity in terms
of word
replacement~\citep{dixon2018measuring,garg2019counterfactual,yurochkin2020sensei,liang2020towards}.
Given a sentence $s$, a similar sentence $s'$ is generated by replacing each word in $s$,
that belongs to a list of words $A_{j}$ indicative of a demographic group $j$,
by a word from list $A_{j'}$, indicative of another demographic group $j' \neq j$.
This approach has
several limitations: (i) it relies on having exhaustively curated word lists $A_{j}$ of
sensitive terms, (ii) perturbations that cannot be represented by replacing single sensitive terms are not covered, and (iii) many terms are only
indicative of demographic groups in specific contexts, hence directly replacing them
with other terms will not always result in a similar pair $(s, s')$ according to human intuition.
Indeed, word replacement rules can often produce sentence pairs that only differ in an axis not relevant to fairness (\eg by replacing ``white house''
with ``black house''). In addition, they can generate so-called
\textit{asymmetric  counterfactuals}~\citep{garg2019counterfactual}:
sentence pairs $(s,s')$ that look similar but clearly do not warrant similar treatment. For example,
in the context of toxicity classification, the text
``I don't like this movie. It is so old'' may not be considered toxic while
``I don't like this movie. It is so gay'' clearly is.

\paragraph{This work: generating fairness specifications for text classification} The central
challenge we consider in this work is how to generate a diverse set of input pairs that aligns with human intuition about which inputs should be treated similarly in the context of a fixed text classification task. These pairs then induce fairness constraints that collectively define an implict fairness specification on a downstream classifier, as individual fairness postulates that they should be classified in the same way.

We address this challenge via a three-stage pipeline, summarized in \cref{fig:Main}.
First, we start from a training dataset $D$ for the text classification task under consideration and generate a set $C^w$ of
candidate pairs $(s,s')$ by applying word replacement to sentences $s\in D$.
Second, to improve diversity and expand on word replacement rules,
we extend $C^w$ to a larger set of pairs $C^e$ by borrowing
unsupervised style transfer ideas. We change markers of demographic groups, \eg
``women'', ``black people'' or ``Christians'', in sentences $s\in D$ by replacing  the style classifier used by modern unsupervised style
transfer methods ~\citep{reid2021lewis, lee2020stable} with a classifier trained to identify mentions of demographic groups. In addition, we add pairs from GPT-3~\citep{brown2020language}, prompted to change markers of demographic groups for sentences in $D$ in a zero-shot
fashion.
\begin{figure}[t]
	\centering
	\includegraphics[width=\textwidth]{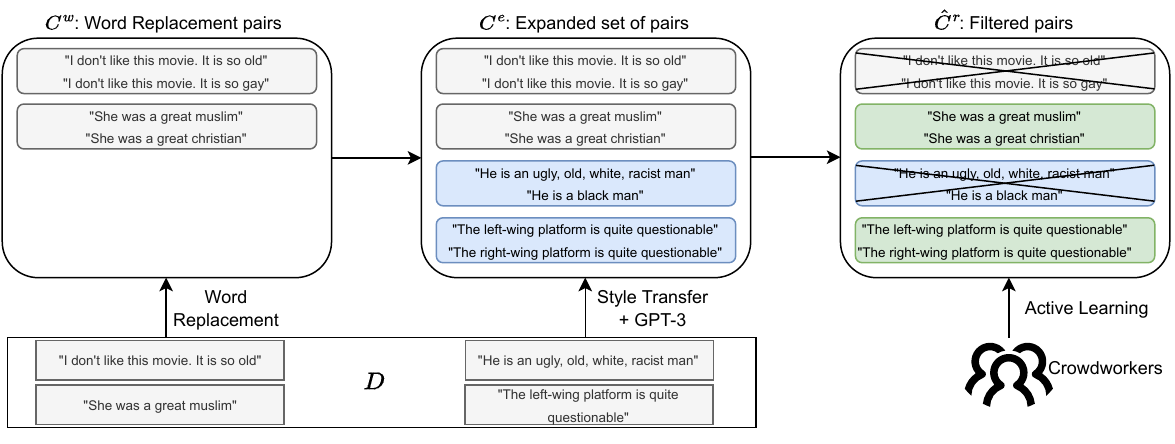}
	\caption{
		Workflow overview.
		We begin by generating sentence pairs using word replacement, and then add
		pairs of sentences leveraging style transfer and GPT-3. Then, we use active learning
		and crowdworker judgments to identify pairs that deserve similar treatment according to human intuition.
	}
	\label{fig:Main}
\end{figure}
Third, to identify which of the generated pairs align with human intuition about fairness in the context of the considered classification task, we design a crowdsourcing experiment in which workers are presented with candidate pairs and indicate if the pairs should be treated similarly for the considered task or not. Since obtaining human feedback is expensive, we label a small subset of the generated pool and train a BERT-based \citep{devlin2018bert}
classifier $\hat{\varphi}$ to recognize pairs that should be treated similarly, yielding a final set of filtered pairs $\hat{C}^{r} \subseteq C^e$. To further reduce labeling costs, we use
active learning similar to~\citep{griesshaber2020fine} to decide which pairs to label. We also demonstrate that the final set of constraints $\hat{C}^{r}$ can be used for training fairness-aware downstream classifiers, by adopting the Counterfactual Logit Pairing (CLP) regularizer of \citep{garg2019counterfactual}.

While our pipeline can in principle be used in the context of most text classification tasks, we instantiate it in the context of toxicity classification. Our experimental results, based on a large dataset for online content moderation, show that in this context our pipeline effectively generates a set of candidate pairs that covers more diverse perturbations than existing word replacement based approaches and successfully leverages human feedback to verify and filter these candidate pairs.


\paragraph{Main contributions} We make the following contributions:
\begin{itemize}[noitemsep,topsep=0pt]
  \item We introduce a method for generating datasets of diverse candidate pairs for individual fairness specifications. Towards that, we leverage GPT-3 and unsupervised style transfer to modify demographic attributes mentioned in sentences.
  \item We show that human feedback can be used for training a classifier that automatically identifies pairs that align with human fairness intuitions for a considered downstream task.
  \item We instantiate our framework in the context of toxicity classification. We experimentally show that the proposed pairs cover more diverse perturbations than  word replacement, that crowdworkers agree with more than $75\%$ of proposed pairs and that our learned approximate specification can effectively be used to train fairness-aware downstream classifiers.
\end{itemize}

%% file: related.tex
\section{Related Work}
\paragraph{Bias in NLP}
Early work on bias in Natural Language Processing has focused on unwanted correlations
between the word embeddings of identifiers for protected demographic groups and
unrelated categories such as
occupations~\citep{bolukbasi2016man,caliskan2017semantics}. More recently,
generative language models have been found to harbor stereotypical
biases~\citep{liang2020towards,nadeem2020stereoset,vig2020investigating,smith2022m}.
Specific to text classification, identity terms such as ``gay'' and explicit indicators
of gender have been shown to significantly impact the outputs of classifiers trained to identify
toxic comments~\citep{dixon2018measuring} or to predict a
person's occupation from their biography~\citep{de2019bias}.
\citet{olteanu2017limits} demonstrate that
human perceptions of the quality of a toxicity classifier can depend on the precise
nature of errors made by the classifier, as well as the annotators' previous
experiences with hate speech. Similarly, \citet{blodgett2020language} recommend
authors to explictly consider why, how and to whom the biases they identify
are harmful.

\paragraph{Language models for data augmentation}
\citet{perez2022red} use a language model to automatically generate test cases for another language model and
\citet{ross2021tailor} automatically create contrast
sets~\citep{gardner2020evaluating} with a language model perturbing sentences
based on control codes.
\citet{rios2020fuzze} use style transfer to change the dialect of African-American
Vernacular English tweets to Standard American English in order to
evaluate the sensitivity to dialect of offensive language detectors, but do not extend
style transfer to mentions of demographic groups.
\citet{hartvigsen2022toxigen} use language models to
generate a balanced dataset of benign and toxic comments about minority groups and
demonstrate that finetuning a toxicity classifier on this dataset can
substantially limit its reliance on spurious correlations between identity terms
and toxicity. However, their dataset is non-parallel, hence it cannot be used for
evaluating individual fairness. Meanwhile, \citet{qian2022perturbation} train a
perturber model to imitate human rewrites $s'$ of comments $s$ that aim to modify
mentions of demographic groups, and demonstrate that finetuning language models
on the modified comments reduces demographic biases. Although this approach creates
parallel data, it is limited by its reliance on large amounts of expensive
human rewrites, which is likely why the authors only use it for perturbations
along given demographic axes such as gender. In contrast, we allow for
perturbations across axes and only require human annotations rather than rewrites.

\paragraph{Learning fairness notions from data}
\citet{ilvento2019metric} provides an algorithm to approximate arbitrary individual
fairness metrics for $N$ datapoints in $O(N\log N )$ queries, which can be
practically infeasible.
Meanwhile, \citet{mukherjee2020two} suggest training a classifier to predict binary
fairness judgments on pairs $(s,s')$ in order to learn a fairness metric $\phi$,
but restrict themselves to Mahalanobis distances on top of a feature representation
$\xi(s)$, limiting their expressive power.
In contrast to our work, these works do not validate their learned fairness notions with human feedback.
To that end, \citet{cheng2021soliciting} present an interface to holistically elicit
stakeholders' fairness judgments, whereas~\citet{wang2019empirical} aim to learn a
bilinear fairness metric for tabular data based on clustering human
annotations.
 Another
strain of work aims to directly learn fair classifiers
without an explicit fairness metric: given access to similarity queries,
\citet{jung2019algorithmic} propose an algorithm with generalization bounds
for fairness and accuracy that requires polynomially many queries to cost-sensitive
classification oracle, while other
work~\citep{gillen2018online,bechavod2020metric} focuses on online learning of
individually fair models. Lastly, \citet{lahoti2019operationalizing} use examples
of similar pairs $(s,s')$ to directly learn a representation that aims to ensure
geometric similarity for similar pairs while preserving nearest neighbors in the
input space. This approach is difficult to use for non-tabular data,
in which nearest neighbor relations do not necessarily carry semantic meaning.
In contrast to these works, we are not only interested in training fair classifiers,
but also aim to learn the similarity function which approximates human
intuitions about fairness for the task.

\paragraph{Enforcing fairness constraints}
\citet{garg2019counterfactual} suggest enforcing fairness constraints via censoring terms indicative of demographic groups, and by extending logit pairing \citep{kannan2018adversarial} to
counterfactual logit pairing (CLP): during training, a classifier $f$ with
logits $l$ is regularized by the term $\lambda ||l(s)-l(s')||_{2}$ for similar datapoints $s$ and $s'$.
\citet{yurochkin2019training} and \citet{yurochkin2020sensei} use distributionally
robust optimization and transport-based regularization respectively to train a
toxicity classifier with distributional fairness guarantees for bilinear fairness
metric similar to the ones from~\citep{mukherjee2020two}.
\citep{ruoss2020learning,yeom2020individual,peychev2021latent}
not only enforce, but also certify the adherence to individual
fairness constraints expressed in logical formulas, weighted $L^{p}$ metrics
or similarity sets defined in the latent space of a generative model.
However, except for CLP and censoring, all of these methods require a known similarity metric with
a specific functional form, which is not always available in practice.

%% file: method.tex
\newtheorem{proposition}{Proposition}[section]
\section{Method}
This section presents our end-to-end framework for generating candidate pairs for
individual fairness specifications
for a given text classification task,
identifying candidates that indeed represent fairness constraints for that task and using them for training individually fair downstream classifiers.
In \cref{subsection:Transfer} we expand on existing word
replacement definitions of individual fairness in text classification by
generating further candidate constraints.
Next, in~\cref{subsection:learning} we leverage human feedback to learn an approximate similarity function $\hat{\varphi}$ to identify a set of
relevant constraints $\hat{C}^{r} \subseteq C^e$.
Finally, in~\cref{subsection:downstream} we train a fairness-aware classifier $f$ using CLP on the filtered constraint set
$\hat{C}^{r}$.

\subsection{Expanding fairness constraints}\label{subsection:Transfer}
We expand the word replacement based constraint set from~\citep{garg2019counterfactual} by implementing three different ways to modify markers of demographic groups mentioned in a sentence $s$: an extended word replacement list, unsupervised style transfer, and zero-shot modification using GPT-3.

\paragraph{Word Replacement}\label{para:WR} First, we enrich the word replacement method by using the extensive lists of words associated with different protected demographic groups presented in \citep{smith2022m}. The pool of terms is substantially larger than the $50$ identity terms from \citep{garg2019counterfactual}. We modify markers of group $j$ in a comment $s$ by
replacing all words on the respective list of words associated with group $j$ with words from the list associated with the target group $j'$.

\paragraph{Unsupervised Style Transfer}\label{para:ST}
Second, we use an unsupervised style transfer approach based on prototype editing (see~\citep{jin2022deep} for an extensive review on style transfer) to transform markers of a demographic group $j$ in a sentence
$s$ to markers of another demographic group $j'$, creating a new sentence $s'$.
Prototype editing identifies markers $a$ of a source style $A$ in a text $s$, and substitutes
them by markers $a'$ of a target style $A'$. It can achieve unsupervised style transfer with minimal modifications to a source sentence $s$. Importantly, modern prototype editing algorithms rely solely on a style classifier to define their notion of style, so that they can transfer mentions of demographic groups when used with a classifier trained to identify such mentions.

Our approach consists of three phases. First, we train a multi-headed RoBERTa-based \citep{liu2019roberta} classifier $c$ to predict the presence of mentions of demographic groups $j$ in a sentence $s$. Second, following \citep{reid2021lewis}, we train a BART-based \citep{lewis2019bart} group-conditioned generator $g(s_{t},j)$: given a sentence $s$ consisting of $n$ tokens that mentions group $j$, we remove mentions of demographic groups from $s$ by masking tokens at positions $k$ with above-average attention weights $a_{k}\geq \bar{a}$, where $a_{k}$ represents the maximum attention weight at position $k$ in the penultimate layer of $c$ and the average is taken over all token positions for the sentence $s$. After merging consecutive masks, this yields a template $s_{t}$ on which $g(s_{t},j)$ is trained to reconstruct $s$. Third, we modify sentences $s$ that mention group $j$ to instead mention group $j'$ by first creating a template $s_{t}'$ as in \citep{lee2020stable}: we iteratively mask the tokens in $s$ for which masking most reduces the likelihood $p_{c}(j|s_{t}')$ of $j$ according to the group-presence classifier $c$, until it falls below a fixed threshold $T$. Then, we generate $s'$ as $g(s_{t}',j')$ using beam search with width 5 and selecting according to  $p_{c}(j'|s')-p_{c}(j|s')$, the difference in likelihoods assigned to $j'$ and $j$ for $s'$ by $c$.

We use this approach rather than the attention-based masking from \citep{reid2021lewis} for the third step because the attention values $a_{k}$ are shared between the prediction heads of $c$ for all groups $j$. This means that attention-based masking might mask tokens related to a third group $j''$ instead of tokens related to $j$ for sentences $s$ in which multiple demographic groups are mentioned. While unlikely to be very detrimental during the training of the class-conditioned generator $g$, using attention for creating templates $s_{t}$ for $g$ can thus cause group transfer to target the wrong source group $j$.

The unsupervised style transfer approach promises multiple advantages. First, style transfer is likely to reproduce terms encountered during training, helping it to pick up on rare demographic terms that are particular to its training distribution which can be chosen to equal the training distribution for downstream tasks. In addition, unlike concurrent work by \citet{qian2022perturbation}, unsupervised style transfer only requires labels $y_{j}(s)$ indicating the mention of demographic group $j$ in a sentence $s$ rather than large amount of expensive human-produced examples of demographic group transfer. This allows us to modify mentions of demographic groups across axes like gender, religion and race, rather than restricting ourselves to changes within each of these axes.

\paragraph{GPT-3}\label{para:GPT} Lastly, we make use of GPT-3~\citep{brown2020language} to transform markers of protected demographic groups in a zero-shot fashion. We use three methods for generating pairs based on GPT-3. First, we prepend an example $s$ mentioning group $j$ with the prompt "Please rewrite the following sentence to be about $j'$ rather than $j$". Second, we use GPT-3's edit
mode\footnote{https://openai.com/blog/gpt-3-edit-insert/} with a similar prompt. Lastly, we generate candidate modifications $s'$ by word replacement, and postprocess them using GPT-3's edit mode with the prompt "Fix grammatical errors and logical inconsistencies".

While the GPT-3 approach does not automatically adapt to the relevant distribution of demographic terms, it does not require any additional data, or training of language models. To ensure that mentions of demographic group $j$ were indeed replaced by $j'$ going from $s$ to $s'$, we use the same group-presence classifier $c$ as for the unsupervised style transfer approach
to heuristically identify successful group transfer and discard pairs $(s,s')$ for which group transfer failed, for all three of our approaches. Implementation details are described in \cref{appendix:STYLE} and \cref{appendix:examples} contains examples of generated sentences.

\subsection{Learning the similarity function}\label{subsection:learning}
In order to evaluate to what extent the proposed similarity criteria align with human intuition, we leverage human feedback, via a crowdsourcing study described in more detail in \cref{sec:experiments}, to obtain labels $\varphi(s, s')$ which indicate whether the pair $(s,s')$ should be treated similarly for the sake of individual fairness ($\varphi(s, s') = 0$) or not ($\varphi(s, s') = 1$). In particular, identifying which pairs align with human labelers' intuition about fairness can
help detect asymmetric
counterfactuals, as well as failed attempts at style transfer for which $s'$
cannot be interpreted as a meaningful modification of $s$.

Since labeling all of $C^e$ can be prohibitively expensive, we train a probabilistic model $p_{\hat{\varphi}}(s,s')$ on a labeled subset of $C^e$ and use it to predict $\varphi(s,s')$ for the remaining pairs $(s,s')$. The similarity function $\varphi$ that specifies individual fairness is then approximated as $\hat{\varphi}(s,s') := 1 \Leftrightarrow p_{\hat{\varphi}}(s,s')>t$ for a given classification threshold $t$. Instead of using bilinear logits based on
features for both $s$ and $s'$~\citep{mukherjee2020two}, we tokenize $s$ and $s'$
and train a BERT-based classifier on the concatenated tokens. This allows for a more holistic comparison
between $s$ and $s'$ as attention heads can directly attend to differences
between $s$ and $s'$ in earlier layers (see~\cref{appendix:active_appendix} for
more details).
\paragraph{Active Learning from human fairness judgments}\label{para:active}
To make optimal use of costly human queries, we employ active learning when training the
classifier $\hat{\varphi}$. We use the variation ratios $ 1-\max_{y}p(y|x)$ to select the data points with the largest uncertainty about the correct label, an approach that is often dubbed Least Confidence (LC) and estimate $p$ using a Dropout-based Monte-Carlo estimate  \citep{gal2016dropout,gal2017deep}. As in \citep{griesshaber2020fine}, we aim to save resources by precomputing features for the BERT-part of $\hat{\varphi}$ and performing Monte-Carlo dropout on the classification head of $\hat{\varphi}$ only. Concretely, after training $\hat{\varphi}$ on an initial randomly selected
dataset $D_{0}\subset C^e$ with labels $\varphi(s,s')$, we iteratively select new unlabeled training data $D_{i}\subset C^e \setminus \bigcup_{j<i} D_{j}$
with $|D_{i}|=1000$, based on the variation ratios, query labels for $D_{i}$, and retrain $\hat{\varphi}$ on $D_{i}$. As different annotators can disagree about whether or not two
sentences $s$ and $s'$ should be treated similarly, we use a majority vote for evaluation. Inspired by ~\citet{chen2022clean}'s approach for dealing with noise in crowdsourcing, we use a single human query per pair $(s,s')$ during active learning, and relabel pairs that are especially likely to be mislabeled after active learning has concluded.

\subsection{Training a fair(er) classifier}\label{subsection:downstream}
Finally, we train a fairness-aware classifier by accounting for the constraints defined by the learned similarity function. Specifically, we define the filtered constraint set
$\hat{C}^{r}=\{(s,s')\in C^e: \hat{\varphi}(s,s')=0\}$. We then train a RoBERTa-based \citep{liu2019roberta} downstream
classifier $f$, empirically enforcing the constraints implied by $\hat{C}^{r}$ by
using the Counterfactual
Logit Pairing (CLP) regularizer $\lambda \sum_{s,s': \phi(s,s') = 0}||l(s)-l(s')||_{2}$ of \cite{garg2019counterfactual}. Here, $l$ represents the logits of the classifier $f$. If $\hat{\varphi}$ accurately approximates human fairness intuitions, this approach avoids enforcing
constraints implied by asymmetric counterfactuals $(s,s')$ (pairs with $\varphi(s,s')=1$)
while properly enforcing actual constraints (pairs with $\varphi(s,s')=0$).
Further training details can be found
in~\cref{appendix:filtering_appendix}.

%% file: experiments.tex
\section{Experiments}
\label{sec:experiments}

In this section, we experimentally evaluate our framework in the context of toxicity classification. Our key
findings are: (i) the pairs generated by our method cover a wider range of perturbations compared to word replacement pairs only (\cref{sec:diversity_results}), while mostly aligning with human intuition about  individual
fairness in toxicity classification (\cref{sec:relevance_results}); (ii) the underlying similarity function $\varphi$ can be approximated by
active learning from human judgements (\cref{sec:learning_results}), and (iii) the produced constraints can be
used to enforce individual fairness on a downstream toxicity classifier (\cref{sec:robustness_results}).


\subsection{Dataset and setup}\label{section:Dataset}
We focus on toxicity classification on the Jigsaw Civil Comments
dataset\footnote{https://www.kaggle.com/competitions/jigsaw-unintended-bias-in-toxicity-classification/data}.
The dataset contains around 2 million online comments $s$, as well as labels $toxic(s)$ indicating the fraction of human labelers that considered comment $s$ toxic. We define binary classification labels $y(s) := toxic(s)>0.5$. A subset $D$ of the Civil Comments dataset also contains labels $A_{j}(s)$ that indicate the fraction of human labelers that think comment $s$ mentions the demographic group $j$. We again define binary classification labels as  $y_{j}(s) := A_{j}(s)>0.5$ for these comments, and use them to train our group-presence classifier $c$. We only consider the subset $D' \subset D$ for which no nan-values are contained in the dataset, and the RoBERTa-tokenized version of $s$ does not exceed a length of 64 tokens. We furthermore split $D'$ into a training set containing $75\%$ of $D'$ and a test set containing the other $25\%$.

To build the pool $C^e$ of candidate pairs, for word replacement and style transfer, we attempt to produce modified comments $s_{j'}'$ mentioning group $j'$ for each $s \in D'$ for all demographic groups $j$ with $y_{j}(s)=1$ and all possible target groups $j'$. For GPT-3, we use a subset of $D'$ due to limited resources. We then combine \num[group-separator={,}]{42500} randomly selected pairs $(s,s')$ with $s$ in the training part of $D'$ for word replacement and style transfer each and a total of \num[group-separator={,}]{15000} pairs $(s,s')$ for our three GPT-3 approaches, to form the set of candidate constraints $C^e$. We similarly construct a set of test constraints of a fourth of $C^e$'s size from the test portion of $D'$. More technical details can be found in \cref{appendix:STYLE}.

Throughout this paper, whenever we report individual fairness for a classifier, we refer to the proportion of pairs $(s,s')$ in a test pool of similar pairs for which $f(s)=f(s')$ rather than $f(s)\neq f(s')$. This metric captures that every instance of treating similar pairs differently is harmful. It is also a natural upper bound for certified individual fairness \citep{ruoss2020learning,peychev2021latent} and prediction consistency \citep{yurochkin2020sensei} which consider equal predictions across all comments $s’$ similar to a given comment $s$ for simpler formal specifications of similarity.

\subsection{Coverage of generated fairness constraints}
\label{sec:diversity_results}
To validate that our approach covers a wider range of
perturbations than word replacement, we train $4$ different toxicity classifiers,
using CLP with different constraint sets: our full constraint set
$C^e$, as well as $C_1, C_2, C_3$ of the same size as $C^e$. The pairs in $C_1$,
also referred to as WR$_{50}$,  were generated by word replacement
based on the $50$ identity terms from \citep{garg2019counterfactual} \footnote{We did not discard any pairs based on a classifier for $C_1$.};
the pairs in $C_2$, also referred to as WR,  were generated by
word replacement by using the larger list of terms of \citep{smith2022m}; and
the pairs in $C_{3}$, also referred to as ST, were created
by the style transfer method. We cross-evaluate the performance of $4$
classifiers trained with CLP with $\lambda=5$,
using each of the constraint sets $C_{i}$ in terms of 
test-time individual fairness measured as
the proportion of similar pairs $(s,s')$ for which $f(s)=f(s')$ in $C^e$ and
each of the $4$ constraint sets $C_{i}$,
and balanced accuracy. \Cref{table:robtrans} reports
these numbers and the performance of a ``baseline'' model,
trained without CLP. \mycomment{Do we actually need the Ci,
or can we just use WR, WR50 and ST?}

The results in \cref{table:robtrans} indicate that each classifier achieves high individual fairness when evaluated on test constraint pairs corresponding to the constraints used during its training (\textit{in italics}) but performs worse when evaluated on other constraint pairs. This indicates that enforcing the similarity notions corresponding to different pair-generation methods can produce substantially different classifiers and that the adherence to individual fairness does not perfectly generalize across our generation methods. We note that training with CLP on $C^e$ or our style transfer pairs $C_{3}$ does not just yield significantly improved constraint adherence on $C_{3}$, but also generalizes well to $C_{1}$ and $C_{2}$ (see the numbers in bold), without losing much downstream accuracy. More details can be found in \cref{appendix:STYLE,appendix:filtering_appendix}.
\begin{table}[ht]
	\begin{center}
		\begin{tabular}{@{}lccccc@{}}
			\toprule
			Training/Evaluation &  BA & WR$_{50}$ ($C_{1}$) & WR ($C_{2}$)  & ST ($C_{3}$) & Full $C^e$  \\
			\midrule
			Baseline (no CLP) & $ 88.4 \pm 0.1  $& $ 78.4 \pm 1.4 $ & $ 81.3 \pm 1.5 $  & $ 76.7 \pm 1.8 $ &  $ 78.5 \pm 1.5$   \\
			\midrule
			CLP WR$_{50}$($C_{1}$)  & $ 87.0 \pm 0.3 $ & $ \mathit{98.3 \pm 0.1} $ & $ 89.1 \pm 1.9 $& $86.3 \pm 1.9$  &  $ 87.3 \pm 1.8 $  \\
			CLP WR ($C_{2}$) & $ 87.2 \pm 0.1 $& $ 93.1\pm 1.2$ & $\mathit{98.2\pm 0.4}$ & $90.5 \pm 1.7 $ &  $92.9\pm 1.2$    \\
			CLP ST ($C_{3}$) & $ 85.9 \pm 0.1 $& $ \mathbf{95.3\pm 0.4}$ & $\mathbf{97.1 \pm 0.3}$ & $\mathit{95.4 \pm 0.4}$  &  $95.5\pm 0.3$     \\
			CLP  Full $C^e$  & $ 85.0 \pm 3.4 $&  $\mathbf{95.5 \pm 0.9} $ & $ \mathbf{97.8 \pm 0.6} $ & $94.9 \pm 0.9$ & $ \mathit{95.7 \pm 0.8} $  \\
			\bottomrule
		\end{tabular}
	\end{center}
	\caption{Balanced accuracy and individual fairness
	evaluated on different $C_{i}$
	(proportion of  pairs $(s,s')\in C_{i}$ for which $f(s)=f(s')$)
	for a RoBERTa-based toxicity classifier $f$ trained
	with CLP using different constraint sets $C_{i}$ in
	each row. Results are averaged over 5 runs and $\pm$ indicates the
	difference from the bounds of a naive $95\%$ confidence interval assuming
	normally distributed errors. \mycomment{Ci is abouse of notation, but it
	should still be clear that this includes C here?}}\label{table:robtrans}
\end{table}
\vspace{-2.5mm}
\subsection{Relevance of generated fairness constraints}
\label{sec:relevance_results}

To validate that the fairness constraints we generated are relevant and intuitive, we conducted a human evaluation using Amazon's MechanicalTurk. Workers were presented with a pair $(s,s')$ consisting of a comment $s$ from the Civil Comments dataset, as well as a modified version $s'$ and asked about whether they believe that the two comments should be treated similarly. Treatment was framed in terms of toxicity classification for content moderation, ensuring that we verify the relevance of the learned notions relevant to this specific task. Workers were also asked whether the demographic group was transferred correctly from a given $j$ to a given $j'$ and whether the content of $s$ has been preserved in $s'$ apart from the demographic group transfer.  Further details can be found in \cref{appendix:mturk}.

We collected human feedback for a set $S$ containing a total of $720$ pairs $(s,s')$ with $240$ each  produced by our style transfer approach, GPT-3 in a zero-shot fashion, and word replacement using the list from \citep{garg2019counterfactual}\footnote{We again did not discard pairs based on a classifier for these pairs.}. The $240$ pairs per method were split into $80$ pairs for each of the axes male$\leftrightarrow$female, christian$\leftrightarrow$muslim and black$\leftrightarrow$white, half of which with $s$ mentioning the first demographic group and half of them with $s$ mentioning the second group. Each pair $(s,s')$ was shown to nine workers and responses aggregated by a majority vote. \Cref{table:Summary methods} reports how often workers affirmatively answered the three questions from the previous paragraph for different methods.

\begin{table}[t]
	\begin{center}
		\begin{tabular}{ @{}lcccc@{} }
			\toprule
			Generation Method & Different Treatment Unfair &Group Transfer & Content Preservation \\
			\midrule
			Word replacement  & 85.9 (97.5) &  89.3 (95.0)   & 88.1 (100)  \\
			Style Transfer  & 85.2 (96.2) & 79.2 (85.4) & 79.2 (91.2)\\
			GPT-3 &  83.2 (93.7) & 81.9 (89.5)  & 78.4 (87.9) \\
			\bottomrule
		\end{tabular}
	\end{center}
	\caption{Crowdsourced answers to questions about comment
	pairs $(s,s')$ for different methods for demographic group
	transfer: do the two comments deserve similar treatment, was
	the demographic group successfully transferred, and was content preserved
	apart from that? The first number represents the percentage of the answer
	across all queries, while the second number (in brackets) represents the
	percentage of comment pairs for which the answer was the majority vote across
	$9$ queries.} \label{table:Summary methods}
\end{table}

The results in \cref{table:Summary methods} demonstrate that all three methods produce relevant fairness constraints, according to a majority of annotators. At the same time, the workers' feedback indicates that the methods were mostly successful at modifying the mentioned demographic group, and at preserving content.
While word replacement generally performs better in terms of group transfer and content preservation, this only translates to a small advantage in terms of producing pairs that represent actual fairness constraints $(\varphi(s,s')=0)$. See \cref{table:Summary methods full,table:Summary axis} for more detailed results.

\subsection{Learning the similarity function}
\label{sec:learning_results}
Since labeling pairs through human feedback is costly, obtaining labels for all candidate pairs in $C^e$ can be prohibitively expensive. Therefore, we employed our active learning approach to efficiently train our classifier $\hat{\varphi}$ from relatively few human judgments, with the goal of using it to identify pairs that represent actual fairness constraints on the remaining pool of candidates. We conducted $6$ steps of active learning with $1000$ queries each, selected by the LC criterion. Failed queries were discarded, so that we ended up with  $5490$ labeled pairs $((s,s'),\varphi(s,s'))$. Details can be found in \cref{appendix:active_appendix}.

We evaluate our learned classifier on a test set $T$ consisting of $500$ randomly selected pairs from $C^e$ for which five annotators were asked to predict the American's fairness judgment. We labeled pairs based on whether the majority thought they should be treated similarly ($\varphi(s,s') = 0$), or not ($\varphi(s,s') = 1$). Because $78.8\%$ of the pairs $(s,s')$ in $T$ represented fairness constraints $(\varphi(s,s')=0)$, we report Balanced Accuracy (BA), in addition to standard accuracy (ACC) and the true positive and negative rates (TPR and TNR). \cref{table:AL R} displays these metrics for classifiers resulting from our active learning method for different classification thresholds $t$ and with and without subsequent relabeling.

\begin{table}[ht]
	\begin{center}
		\noindent
		\begin{tabular}{@{}lcccc@{} }
			\toprule
			Method & ACC & TNR & TPR & BA\\
			\midrule
			Constant Baseline & 78.8 & $\mathbf{100.0}$ & 0.0 & 50.0 \\
			\midrule
			Active Learning t=0.5 & $ 79.8 \pm 0.3$  & $ 97.2 \pm 0.3 $ & $ 15.1 \pm 1.2 $ & $56.1$  \\
			Active Learning + Relabel t=0.5 & $\mathbf{81.1 \pm 0.3}$ & $95.5 \pm 0.7$ & $  28.6 \pm 2.2$ & $62.0$ \\
			\midrule
			Active Learning t=0.1 & $ 80.0 \pm 0.5 $  & $ 95.2 \pm 0.7$ & $ 23.7 \pm 3.5 $ & $59.4$\\
			Active Learning + Relabel t=0.1 &$ 80.7 \pm 0.6 $ & $ 93.0 \pm 0.9 $ & $ 35.0 \pm 1.3 $ & $64.0$ \\
			\midrule
			Active Learning t=0.01 &$78.7 \pm 1.1 $ & $87.5 \pm 2.1$ & $45.7 \pm 1.8$ & $66.6$  \\
			Active Learning + Relabel t=0.01 & $78.3 \pm 0.7$ &$86.8\pm 1.5$ & $\mathbf{46.6\pm 2.5}$ & $\textbf{66.7}$ \\
			\bottomrule
		\end{tabular}
	\end{center}
	\caption{Performance
		(in terms of accuracy, true negative/positive rate
			and balanced accuracy) for
		similarity classifiers $\hat{\varphi}$
		trained on human fairness judgments with
			and without relabeling, evaluated on the test set $T$
		with different decision thresholds $t$.
		Results are averaged over $10$ repetitions of training on the data labeled in the last step/the relabeled data.
		$\pm$ indicates the difference from the upper/lower bound of a naive $95\%$
		confidence interval assuming normally distributed errors.}\label{table:AL R}
\end{table}
We observe that $\hat{\varphi}$ performs substantially better than random, achieving BA of $66.7\%$ when used with an aggressive classifier threshold $t$. \Cref{table:AL R} also validates our relabeling approach: after observing that our classifier was biased towards predicting $\hat{\varphi}(s,s')=0$ on a held-out validation set, we collected two additional labels for $500$ pairs $(s,s')$ for which both the human and the predicted label were equal to zero, selected based on the LC criterion. The majority vote over all three annotators was $\varphi(s,s')=1$ for $47 \%$ of these pairs, showing that our approach correctly identified pairs that were likely to be mislabeled. Retraining our classifier on the updated majority votes also substantially increased TPR at little costs to TNR, especially for balanced classification thresholds $t$ close to $0.5$. According to a qualitative evaluation, many sentence pairs $(s,s')$ predicted to not represent fairness constraints $(\hat{\varphi}(s,s')=1)$ had the words "boy" or "man" replaced by terms denoting identity membership. Such sentence pairs, like "You boys don't go fishing when you go on those vacations, do you?" and "You Hindus don't go fishing when you go on those vacations, do you?" were often not seen as fairness constraints by human annotators, as the use of the identity term can be read as mocking. $\hat{\varphi}$ also identified sentence pairs $(s,s')$ for which $s'$ was unrelated to $s$, that were sometimes produced by GPT-3, as not representing fairness constraints. Further results can be found in \cref{appendix:active_appendix}.
\subsection{Training a fairer downstream classifier}
\label{sec:robustness_results}
Lastly, we evaluate whether the pairs $\hat{C}^{r}$ obtainted
by filtering $C^e$ with $\hat{\varphi}$
(trained with relabeling, threshold $t=0.5$)
can help with learning an individually fair
downstream classifier, by training a RoBERTa-based toxicity classifier
$f$ using CLP with $\lambda=5$. More details can be
found in \cref{appendix:filtering_appendix}.
We train toxicity classifiers with CLP using constraint sets defined by
word replacement ($C_{1}$ and $C_{2}$ as in \cref{sec:diversity_results})
and using all of $C^e$, or the filtered version  $\hat{C}^{r}$.
Additionally, we train on a challenging set of constraints, $C^e_{adverse}$
that consists of $C^e$ and \num[group-separator={,}]{10000} adversarially selected pairs ($s,s'$)
created by randomly selecting comments $s$ with toxicity label $y(s)=1$
and randomly selecting comments $s'$ with label $y(s)=0$ from $D$, and
a filtered version $\hat{C}_{adverse}^{r}$ of $C^e_{adverse}$ using
threshold $t=0.5$.

We then evaluate these classifiers in terms of Balanced Accuracy (BA) and individual fairness on the resulting classifiers on the test set $T$. \cref{table:endresults}
shows that compared to word replacement our expanded constraint set $C^e$ consistently yields better adherence to human-validated fairness constraints at the cost of a small drop in BA. However, we do not find a clear improvement from using the filtered constraint set $\hat{C}^{r}$ over the full set of constraints $C^e$. We hypothesize that this is due to our classifier $\hat{\varphi}$'s limited True Positive Rate combined with $\varphi(s,s')$ equalling zero for most pairs $(s,s') \in C^e$ according to human annotators, such that even filtering with a perfect classifier $\hat{\varphi}$ might be of limited utility as most constraints in $C^e$ are indeed relevant. This is supported by our results for $C^e_{adverse}$, where filtering substantially improves BA. Further experiments can be found in \cref{appendix:filtering_appendix}.

\begin{table}[ht]
\begin{center}
\begin{tabular}{ @{}lccc@{} }
\toprule
Method & BA & Fairness ($T$) \\
\midrule
Baseline no CLP & $ \mathbf{88.2 \pm 0.4} $ &  $ 82.1 \pm 2.1 $ \\
\midrule
CLP WR$_{50}$ $(C_{1})$  & $  87.1 \pm 2.0  $ & $ 92.8 \pm 0.9  $  \\
CLP WR $(C_{2})$ & $ 87.2 \pm 0.2 $ & $ 95.8 \pm 0.9$   \\
CLP Full constraint set $C^e$ &  $85.9 \pm 0.3$ & $96.5 \pm 1.4$ \\
CLP Filtered constraint set $\hat{C}^{r}$  & $85.9 \pm 0.5$ & $\mathbf{97.4\pm 1.1}$\\
\midrule
CLP  $C^e_{adverse}$ &  $71.1 \pm 17.4$ &$97.8 \pm 2.2$   \\
CLP  Filtered $\hat{C}_{adverse}^{r}$ &  $79.3 \pm 2.2$ &$98.7 \pm 0.6$ \\
\bottomrule
\end{tabular}
\end{center}
\caption{Balanced accuracy and individual fairness (proportion of human-validated
similar pairs $(s,s')\in T$ with $\phi(s,s')=0$, for which $f(s)=f(s')$)
for toxicity classifiers $f$ trained with CLP and different sets of similar pairs.
Results are averaged over 5 training runs. $\pm$ indicates the difference
from the upper/lower bound of a naive $95\%$ confidence interval assuming
normally distributed errors.}\label{table:endresults}
\end{table}

%% file: conclusion.tex
\section{Conclusion}
\makeatletter
\def\blfootnote{\xdef\@thefnmark{}\@footnotetext}
\makeatother
We proposed a framework for producing expressive and intuitive specifications for individual fairness in text classification. We experimentally demonstrated that our pairs are more expressive than word replacement pairs and that most of the generated pairs were relevant in the context of toxicity classification according to human annotators. We also trained a classifier that automatically identifies relevant pairs and showed that our approach can improve the fairness of a toxicity classifier. 
\blfootnote{\hspace{-0.65cm} Our dataset of fairness judgments is available at \url{https://github.com/eth-sri/fairness-feedback-nlp}}

%% file: ethics.tex
\section{Ethics Statement}
Our human evaluation experiments involving workers from Mechanical Turk were reviewed and approved by the ETH Zurich Ethics Commission as proposal EK 2022-N-117. Workers on Mechanical Turk were warned that they might be shown offensive comments as part of our study and were able to opt out of participating in our study at any time. We also made sure that the per-task compensation was sufficiently high to result in a hourly compensation exceeding the US federal minimum wage. More details on our human evaluation experiments can be found in \cref{appendix:mturk}.

While we believe that our results show that learning more precise fairness notions by involving human feedback is a very promising area of research, we caution against directly using the labels from our human evaluation study $\phi$ for evaluating fairness in high-stakes real-world applications of toxicity classification. First, our results show that there is substantial disagreement between different survey participants about which pairs $(s,s')$ require equal treatment by a fair classifier. While resolving these disagreements via a majority vote is a natural choice, other approaches may be desired in some contexts, for example enforcing equal treatment whenever at least one participant believes it is required or explicitly accounting for distributional information via jury learning \citep{gordon2022jury} or multi-annotator architecture \citep{davani2022dealing}. \mycomment{, especially considering that differential treatment of a similar pair $(s,s')$ might be more costly than equal treatment of a non-similar pair} Second, our survey participants are geographically biased to the US and are neither direct stakeholders, nor experts in discrimination law and hate speech. 

Given that our learning approach shows promising signs of being able to improve upon existing approaches to quantifying individual fairness despite large amounts of disagreement, which is likely to be less common for actual stakeholders and experts, we recommend using it in conjunction with fairness judgments provided by application-specific experts and stakeholders. The collection of additional application-specific fairness labels $\phi$ is especially important when our framework is applied to downstream tasks other than toxicity classification, for which human agreement with the relevance of generated pairs to individual fairness could theoretically be substantially lower than indicated by our study. In addition, the validity of collected labels should be reassessed periodically in order to account for shifts in linguistic, cultural and social contexts \citep{aroyo2019crowdsourcing}.

In addition, our generated pairs only cover perturbations along a fixed list of demographic groups such that fulfilling individual fairness on these pairs does not guarantee invariance to other forms of perturbations. Correspondingly, evaluation on these pairs can help to provide better evidence about the extent of demographic biases but not necessarily other forms of bias present in a text classifier. 

We believe that it is exceedingly difficult to capture human intuitions about individual fairness in complex domains like text using simple rules. Hence, our work implicitly defines similarity based on human-labeled data, which can be difficult to interpret. To instill further in the derived specifications, future work could aim to extract more interpretable specifications from the human feedback, for example using techniques from explainable machine learning. 

%% file: reproducibility.tex
\section{Reproducibility Statement}
We provide code to reproduce our generation pipeline and our experiments on synthetic data, as well as our dataset of human fairness judgments at \url{https://github.com/eth-sri/fairness-feedback-nlp}. 
All of our experiments involving transformer language models use the huggingface transformers library \citep{wolf2020transformers}. Additional details on our human evaluation are provided in \cref{appendix:mturk}.

%% file: acknowledgements.tex
\section*{Acknowledgements}
We would like to thank  Dan  Hendrycks,  Mislav  Balunovi\'{c},  Afra Amini and Dominik Stammbach for helpful comments and discussions during early stages of this work. We also thank the anonymous reviewers for their insightful comments and constructive feedback that helped to improve this paper.

Florian Dorner is grateful for financial support from the Max Planck ETH Center for Learning Systems (CLS) received during part of this work.  Nikola Konstantinov's contributions to this publication were made possible by an ETH AI Center postdoctoral fellowship.

%% file: appendix.tex
\counterwithin{table}{section}
\section{Further Details on Human evaluation}\label{appendix:mturk}
In order to participate, workers had to live in the US and be above 18 years old in addition to being experienced with MechanicalTurk (having completed more than $5000$ HITs\footnote{Bundled tasks on MechanicalTurk for which a remuneration is received on completion} and having a good reputation ($97\%$ acceptance rate across all of the worker's HITs). Workers were warned about the potentially offensive content of some of the comments show in the study by the following statement: "Please note that this study contains offensive content. If you do not wish to see such content, please withdraw from the study by leaving this website." and were also told that they could withdraw from the study at any later point: "You may withdraw your participation at any time without specifying reasons and without any disadvantages (however, you will not get paid for the current HIT in case you withdraw before completing it)".

After encountering a high prevalence of bots, malicious workers or workers that fundamentally misunderstood our task instructions during pilot experiments, we had workers pass a qualification test by providing correct answers for nine out of ten queries $\varphi(s,s')$ for pairs that were hand-designed to have a relatively obvious correct answer. We validated these hand-designed pairs in a separate experiment, querying workers about $\varphi(s,s')$ for 11 pairs ($s,s'$), and asking them to verbally explain each of their decisions, paying a total of $\$1.83$. We only included hand-designed pairs in the qualification test if at least eight out of ten workers produced the intended answer during this experiment, and no worker brought forward convincing arguments against this answer being correct.

Workers were paid $\$0.83$ for a battery of 11 queries $\varphi(s,s')$ for the Active learning and the test set $T$ and first $\$1.20$, later $\$1.30$ for 11 queries about all questions presented in \cref{table:Summary methods}. In both cases, workers had to explain one of their answers verbally, and one of the queries was a attention check pair $(s,s')$ with a relatively obvious correct answer constructed in the same way as for the qualification tests. 

Blocks of queries with wrong answers to the attention check question or (partially) incoherent verbal explanations were manually reviewed, and thrown out in case we were not able to find evidence that the worker had correctly understood the task in which case the worker was not paid. Workers were made aware of this by the following statement in the task description:  "To ensure workers pay attention, we include pairs of comments for which there is a clear correct answer at random positions. Please make sure to carefully read every comment and answer diligently, as your HIT will be rejected if your answer for one of these pairs of comments is wrong,  in which case you will not receive a payment for the HIT. We will also ask you to produce an explanation for one of your answers. Please write at least one sentence and make sure your explanation is coherent with the answer you gave, as your HIT might be rejected otherwise in which case you will not receive a payment for the HIT.". According to https://turkerview.com, a tool used by many professional workers on MechanicalTurk, we paid workers an average hourly rate of $\$16.17$, clearly exceeding the US minimum wage. While this is likely an overestimate, as not all workers use turkerview, the hourly rate is so high, that we still exceed the minimum wage for workers taking twice as long as the average worker using turkerview. 

Workers were introduced to the task as follows: "We are investigating fairness in toxic comment classification and are looking for pairs of comments with similar content that should be treated the same by a content moderator (if one of them was removed, the other one should be removed too). You will be shown pairs of comments and be asked whether the comments are similar in various ways and whether you would consider it unfair for a content moderator to remove one of the comments in each pair but not the other.". The queries read as follows for the first task:
\begin{itemize}
    \item Without further context, would the average american consider it unfair or acceptable for a content moderator to delete one of the shown comments but not the other?
    \begin{itemize}
      \item It would be unfair for one of the two comments but not the other to be deleted (Both comments ought to be treated the same way, no matter what)
      \item If comment a) was deleted, comment b) should be deleted as well, but not the other way around. (Deleting comment b) would be more justifiable)
      \item If comment b) was deleted, comment a) should be deleted as well, but not the other way around. (Deleting comment a) would be more justifiable)
      \item It would be acceptable to delete either of the comments without deleting the other. (Treating the comments differently either way would be fine)
  \end{itemize}
\end{itemize}
and as follows for the second one (Placeholders for group terms are presented in \textit{Italics}. The answer summarized in \cref{table:Summary methods} is highlighted in \textbf{bold}):
\begin{itemize}
  \item Without further context, would you consider it unfair or acceptable for a content moderator to delete one of the shown comments but not the other?
  \begin{itemize}
      \item \textbf{It would be unfair for one of the two comments but not the other to be deleted (Both comments ought to be treated the same way, no matter what)}
      \item If comment a) was deleted, comment b) should be deleted as well, but not the other way around. (Deleting comment b) would be more justifiable)
      \item If comment b) was deleted, comment a) should be deleted as well, but not the other way around. (Deleting comment a) would be more justifiable)
      \item It would be acceptable to delete either of the comments without deleting the other. (Treating the comments differently either way would be fine)
  \end{itemize}
  \item What do you expect the average American to answer to the previous question?
    \begin{itemize}
      \item \textbf{It would be unfair for one of the two comments but not the other to be deleted (Both comments ought to be treated the same way, no matter what)}
      \item If comment a) was deleted, comment b) should be deleted as well, but not the other way around. (Deleting comment b) would be more justifiable)
      \item If comment b) was deleted, comment a) should be deleted as well, but not the other way around. (Deleting comment a) would be more justifiable)
      \item It would be acceptable to delete either of the comments without deleting the other. (Treating the comments differently either way would be fine)
     \end{itemize}
   \item Do you consider one of the two comments to be more factually correct than the other?
   \begin{itemize}
      \item Comment a
    \item Comment b
    \item \textbf{Neither comment / too little context to tell}
   \end{itemize}
   \item Do you consider one of the two comments to be more grammatically correct than the other?
      \begin{itemize}
      \item Comment a
    \item Comment b
    \item \textbf{Neither comment}
   \end{itemize}
   \item Is comment a) about \textit{group a} and comment b) about \textit{group b}?
   \begin{itemize}
       \item \textbf{Yes}
       \item No, comment a) is not about \textit{group a}
       \item No, comment b) is not about \textit{group b}
       \item No, neither
   \end{itemize}
   \item Apart from differences related to \textit{group a} and \textit{group b}, are both comments similar in terms of content?
   \begin{itemize}
       \item \textbf{Yes, they are almost the same.}
       \item They are somewhat similar, but differ in some additional details.
       \item There is an important additional difference between the comments' content
   \end{itemize}
\end{itemize}

\cref{table:Summary methods full} shows an extended version of \cref{table:Summary methods} and includes human annotator's answers to additional questions. It shows that the reason why the advantages of word replacement in terms of group transfer and content preservation do not fully translate to producing pairs that represent actual fairness constraints could be due to its worse performance in terms of preserving factuality. Indeed, we found examples in which word replacement transformed "white house" to "black house"; or Obama is referred to as "white" rather than "black" in a modified comment. These pairs were not seen as fairness constraints by most annotators, while also being judged badly in terms of preserving factuality.
\begin{table}[t]
	\begin{center}
		\begin{tabular}{ @{}lcccc@{} }
			\toprule
			Metric/Method & Word replacement & Style Transfer & GPT-3 \\
			\midrule
			Unfair: Average American & 84.9 (97.5) & 84.6 (95.8) & 83.4 (95.0) \\
			Unfair: Own Opinion & 85.9 (97.5) & 85.2 (96.2) & 83.2 (93.7) \\
			Group Transfer & 89.3 (95.0) & 79.2 (85.4) & 81.9 (89.5)\\
			Content preservation & 88.1 (100) & 79.2 (91.2) & 78.4 (87.9) \\
			Same Factuality & 73.0 (84.1) & 76.2 (87.5) & 78.5 (89.1) \\
			Same Grammaticality & 91.2 (99.1) & 92.9 (97.9) & 92.9 (98.3)\\
			\bottomrule
		\end{tabular}
	\end{center}
	\caption{Human evaluation: Answers to questions about comment pairs $(s,s')$ grouped by different methods for demographic group transfer. The first number represents the fraction of the answer across all queries, while the second number (in brackets) represents the fraction of comment pairs for which the answer was the majority vote across $9$ queries.} \label{table:Summary methods full}
\end{table}

\cref{table:Summary axis} shows the results of the human evaluation on our test set $S$ split along the axis of attribute transfer, rather than generation method as in \ref{table:Summary methods}. Along with the results in \cref{table:Summary methods} they show that despite the general agreement about the relevance of the generated fairness constraints, there is  substantial disagreement between annotators when it comes to deviations from the most common answer across all comments. In all cases, the fraction of comments with majority vote equal to that answer is substantially higher than the overall fraction of these votes across all comments and annotators. The same is true for our set $T$ of $500$ randomly selected pairs from $C^e$ for which we only asked about the average American's fairness judgment: $70.9\%$ of the annotations were $\varphi(s,s')=0$, while the same was true for $78.8\%$ of the per-comment pair majority votes.

\begin{table}[ht]
	\begin{center}
		\begin{tabular}{ @{}lcccc@{} }
			\toprule
			Metric/Method &  male$\leftrightarrow$female &  black$\leftrightarrow$white & christian$\leftrightarrow$muslim \\
			\midrule
			Unfair: Average American & 83.5 (96.6) & 82.2 (94.5) & 87.2 (97.0) \\
			Unfair: Own Opinion & 83.5 (96.6) & 82.4 (92.9) & 88.4 (97.9) \\
			Group Transfer & 82.6 (91.6) & 81.6 (86.6) & 86.2 (91.6)\\
			Content preservation & 84.9 (95.4) & 79.5 (92.0) & 81.3 (91.6) \\
			Same Factuality & 75.3 (82.9) & 73.6 (85.0) & 78.8 (92.9) \\
			Same Grammaticality & 90.5 (97.5) & 92.2 (98.3) & 94.3 (99.5)\\
			\bottomrule
		\end{tabular}
	\end{center}
	\caption{Human evaluation: Answers to questions about comment pairs $(s,s')$ grouped along demographic group transfers along different axes. The first number represents the fraction of the answer across all queries, while the second number (in the brackets) represents the fraction of comment pairs for which the answer was the majority vote across $9$ queries.}\label{table:Summary axis}
\end{table}

Our dataset including the pairs generated by our approach and aggregate human fairness judgments can be accessed at \url{https://github.com/eth-sri/fairness-feedback-nlp}.
\clearpage

\section{Further details on style transfer}\label{appendix:STYLE}
All of our experiments involving transformer language models use the huggingface transformers library \cite{wolf2020transformers}. 
\paragraph{Unsupervised style transfer}
To transform markers of demographic groups in sentences, we first finetune a Multi-headed RoBERTa-based \citep{liu2019roberta} classifier $c$ to predict labels $y_{j}$ indicating the presence of markers of a demographic group $j$ from a list of protected demographic groups $J$ in a sentence $s$. We use the population labels ("Black", "Male", "Heterosexual", "Muslim", etc.) that are provided for a subset of the Civil comments dataset. The group-presence classifier $c$ is based on the roberta-base model, followed by a linear layer with $768$ neurons applied to the output embedding of the first token only, a Tanh layer, another linear layer mapping to a single dimension, and a Sigmoid layer. We train $c$ for $3$ epochs with a batch size of $16$ and use the Adam optimizer \cite{kingma2014adam} with learning rate $0.00001$ to optimize the binary Cross Entropy loss, reweighed by relative label frequency in the dataset. \cref{table:C-RESULTS} shows the balanced accuracy on the test set for all demographic groups in the dataset. For our downstream applications of $c$, we restrict ourselves to the demographic groups for which the classifier $c$'s balanced accuracy is above $90\%$. Furthermore, we also exclude the group labeled "mental illness" because the word replacement lists we used lack a clear analogon.

Then, we finetune a BART-based \citep{lewis2019bart} generator $g$ on a mask-filling task on the same data: For every data point $s$, we sample a group from the set of demographic groups $j$ mentioned in $s$, i.e. $\{j:y_{j}(s)=1\}$, skipping sentences $s$ for which no group $j$ meets this criterion. Inspired by \citep{reid2021lewis} we mask all of $s$'s tokens that have an above-average attention value for the 11th layer of the classifier $c$, merge consecutive mask tokens into one, and prepend the name of the sampled group $j$ to the masked sentence before fedding it to the generator $g$. The generator $g$ is then finetuned to reconstruct $s$ using token-wise Cross Entropy. 

The BART-based generator $g$ is trained starting from the pretrained facebook/bart-large model for a single epoch with batch size 4, again using Adam and a learning rate of $0.00001$. For filling in masked sentences, we pick the completion with the largest difference in the classifier $c$'s pre-sigmoid activation for the target and source demographic groups $j'$ and $j$ among candidate sentences produced by a beam search generation using the generator $g$ with width $5$.

To transfer an example $s$ from mentioning group $j$ to mentioning group $j'$, we follow \citep{lee2020stable} and iteratively mask the token for which masking reduces $p_{c}(y_{j}|x)$ the most, until we reach a threshold of $p_{c}(y_{j}|x)< 0.25$. We use this approach rather than the attention-based masking from \citep{reid2021lewis} because of the lack of theoretical motivation for using attention to identify important features \citep{bastings2020elephant}, and because attention scores are the same for all of our model's group-presence prediction heads, rather than specific to a particular group $j$.\footnote{We used attention during the training of $g$, for which dropping out some tokens unrelated to $j$ is less problematic, in order to save resources.} Then, we prepend a verbal representation of label $j'$ to $s$ to form a prompt $p$, and generate a sentence $s'$ as $g(p)$.

\begin{table}[ht]
\begin{center}
\begin{tabular}{@{}lc|cc|cc@{} }
\toprule
Category & BA & Category  & BA & Category & BA \\
\midrule
Male & $96.5$ & Christian & $96.6$ & Physical disability & $54.9$  \\
Female & $97.8$ & Jewish & $98.9$ & Intellectual disability & $54.3$ \\
Transgender & $99.3$ & Muslim & $98.9$ &Mental illness& $98.3$ \\
Other gender & $50.0$ & Hindu & $98.2$ & Black  & $99.2$ \\
Heterosexual & $98.1$ & Buddhist & $99.2$ & White & $99.5$   \\
Homosexual & $99.3$ & Atheist & $99.6$ & Asian & $98.3$ \\
Bisexual & $65.4$ &Other religion& $50.0$ &Latino& $96.6$ \\
Other sexuality & $50.0$ & Other disability & $50.0$  & Other race& $55.5$ \\
\bottomrule
\end{tabular}
\end{center}
\caption{Balanced accuracies of the group-presence classifier $c$ for different labels}\label{table:C-RESULTS}
\end{table}

\paragraph{Word replacement}
Our word replacement approach is based on the list of words provided in \cite{smith2022m}: Given a sentence $s$ mentioning demographic group $j$ and a target attribute $j'$, we replace all words in $s$ that are on the list associated with $j$ with random words from the list associated with $j'$, replacing nouns with nouns and descriptors with descriptors whenever possible, and nouns with descriptors otherwise. The full list of words we used for word replacement is displayed in \cref{table:WR}.

\paragraph{GPT-3}
We accessed GPT-3 using OpenAI's API\footnote{https://openai.com/api/}. For our first approach, we used the "text-davinci-001" version of GPT3 in a zero-shot manner with the prompt: "Please rewrite the following sentence to be about $j'$ rather than $j$:" followed by a new line and the targeted sentence $s$. The second approach was based on the beta-version of GPT-3's editing mode \footnote{https://openai.com/blog/gpt-3-edit-insert/}. Here, $s'$ is produced using the model "text-davinci-edit-001" with the instruction "Rewrite the text to be about $j'$  rather than $j$". Lastly, we used to same model in conjunction with word replacement: First, we generated a candidate sentence $s''$ using the procedure described in the word replacement section. Then, in order to fix issues caused by the context-blindness of the word replacement approach, we postprocessed $s''$ using "text-davinci-edit-001" with the instruction "Fix grammatical errors and logical inconsistencies" to produce $s'$. We used temperature $= 0.7$ and top\textunderscore p$=1$ in all our approaches and used max\textunderscore tokens$=64$ for "text-davinci-001" to control the length of the modified sentence $s'$.

Please refer to the most up-to-date version of OpenAI's usage policy\footnote{(\url{https://platform.openai.com/docs/usage-policies}} regarding content generation using GPT-3.

\paragraph{Post-filtering} For all three approaches, we performed a post-filtering step to reduce the prevalence of unsuccesful attempts at demographic group transfer in our set of constraints $C^e$. Given a pair $(s,s')$ of an original sentence and a modified version, we only include it in our set of constraints $C^e$, if the classifier probability $p_c(y_{j'}|s')$ for label $j'$ is below $0.5$ and the classifier probability $p_c(y_{j}|s')$ for label $j$ is above $0.5$. \newline
As mentioned in \cref{section:Dataset}, we attempt to produce modified comments $s_{j'}'$ mentioning group $j'$ for each $s$ in $D'$ for all demographic groups $j$ with $y_{j}(s)=1$ and all possible target groups $j'$ for word replacement and style transfer. For GPT-3, we attempted a total of $75$ generations for each of our three generation modes per axis pair of demographic groups $(j,j')$ and direction of group transfer, with the source sentences $s$ randomly selected among the sentences with label $j$ in $D'$. For constructing the secondary test set $S$, we attempted more generations for the axes male$\leftrightarrow$female, christian$\leftrightarrow$muslim and black$\leftrightarrow$white, homosexual$\leftrightarrow$heterosexual. The latter axis was left out of $S$ because we found that the rate of successful generations was too limited. We generated a maximum of $2250$ attempts up until a total of $250$ successful generations (post-filtering step passed) for GPT-3's zero-shot mode, a maximum of $750$ until to a total of $100$ successful generations for GPT-3's edit mode, and up until a total of $100$ successful generations for GPT-3 based postprocessing of word replacement. \cref{table:generations} shows the overall amount of generated pairs per method.

\begin{table}[ht]
\begin{center}
\begin{tabular}{ @{}lcccc@{} }
\hline
Generation Method & Total (Train) & Total (Test) & In $C^e$ (Train) & In $C^e$ (Test) \\
\hline
Word Replacement & $980667$ & $331490$ & $42500$ & $10625$ \\
Style Transfer & $681111$ & $229883$ & $42500$ & $10625$ \\
GPT-3 Zero-Shot & $6322$ & $2139$ & $6200$ & $1550$ \\
GPT-3 Edit Mode & $3704$ & $1199$ & $3500$ & $875$ \\
GPT-3 Postprocessing & $5330$ & $1831$ & $5300$ & $1325$ \\
\hline
\end{tabular}
\end{center}
\caption{Amount of generated pairs $(s,s')$ per generation method.}\label{table:generations}
\end{table}

As an additional experiment to validate the increased diversity of our constraint set $C^e$ we train a similarity classifier\footnote{Using the same architecture as for our active learning experiments described in \cref{appendix:active_appendix}} $\hat{\varphi}$, on $C^e$ to distinguish pairs $(s,s')$ generated by word replacement from pairs generated by style transfer or GPT-3. Training on $100000$ examples without label noise, we are able to achieve over $91.6 \%$ test accuracy on a balanced test set, suggesting that there is a meaningful difference between pairs generated by word replacement and the rest of the constraint candidates $C^e$.

\clearpage
\section{Further details on learning similarity functions}\label{appendix:active_appendix}
First, \cref{prop:similarity} below establishes that robustness with respect to a binary similarity function $\varphi$, i.e. $\varphi(s,s')=0 \Rightarrow f(s)=f(s')$, can fully capture the definition of individual fairness as Lipschitz-Continuity proposed by \citet{dwork2012fairness} for deterministic classifiers $f$.
\begin{proposition}\label{prop:similarity}
	Given a metric $d:X \times X \rightarrow \mathbb{R}$, a binary metric $d_{b}: Y \times Y \rightarrow \{0,1\}$ and a constant $L>0$, there exists a similarity function $\varphi:X \times X \rightarrow \{0,1\}$ such that a function $f:(X,d)\rightarrow (Y,d_{b})$ is Lipschitz-Continuous with constant $L$ if and only if  $\varphi(x,x')\geq d_{b}(f(x),f(x'))$ for all $x,x'\in X$.
\end{proposition}
\begin{proof}
	Define $\varphi(x,x'):=\mathbbm{1}\left\{Ld(x,x') \geq 1\right\}$. Then whenever $d_{b}(f(x),f(x'))=1$, we have $d_{b}(f(x),f(x'))=1\leq \varphi(x,x')$ if and only if $d_{b}(f(x),f(x'))\leq Ld(x,x')$. But if $d_{b}(f(x),f(x'))=0$, the Lipschitz inequality is allways true. Now, assume that $f$ is not Lipschitz: Then, there exist $x,x'\in X$ such that $1=d_{b}(f(x),f(x'))> Ld(x,x')$, implying $0=\varphi(x,x')< d_{b}(f(x),f(x')) =1$
\end{proof}

We use a BERT-based classifier that acts on a pair $(s,s')$ by first tokenizing both $s$ and $s'$ and padding the token representation to a length of $64$, concatenating these tokens and feeding the concatenated token representation into a pretrained bert-uncased-base model. We then apply a linear layer with dropout ($p=0.1$) followed by a Tanh layer and a second linear layer with dropout ($p=0.1$) to obtain single dimensional logits, to which a sigmoid layer is applied before computing the binary Cross Entropy loss.  We use BERT rather than more modern models such as RoBERTa \citep{liu2019roberta} and Deberta \citep{he2020deberta}, as we have found it to clearly outperform them for our task, plausibly because BERT uses a next-sentence-prediction task during pretraining, which is structurally similar to our task of comparing two sentences. \cref{table:architectures} demonstrates the advantage of using BERT, as well as concatenating token representations rather than learning based on the difference between separately produced BERT features for both $s$ and $s'$. Unless stated otherwise, our Active Learning approach trains for five epochs on each queried block $D_{i}$ before selecting new data $D_{i+1}$ to label.
\begin{table}[ht]
	\begin{center}
		\begin{tabular}{ @{}lcc }
			\toprule
			Model & BA \\
			\midrule
			BERT-Concat & \textbf{86.7} \\
			BERT-Merge & 79.9 \\
			BERT-Featurediff & 67.8 \\
			DeBERTa-Concat & 54.7 \\
			DeBERTa-Merge & 53.2 \\
			DeBERTa-Featurediff & 50.8 \\
			RoBERTa-Concat & 52.1 \\
			RoBERTa-Merge & 50.3 \\
			RoBERTa-Featurediff & 51.1 \\
			BERT-Large-Concat &84.4 \\
			BERT-Large-Merge & 84.1 \\
			BERT-Large-Featurediff & 59.2 \\
			BERT-Bilinear & 50.7\\
			\bottomrule
		\end{tabular}
	\end{center}
	\caption{Different architectures trained for one epoch on 5000 samples from a set of pairs $(s,s')$ generated using word replacement to distinguish demograpghic group transfer within the same category of gender and sexuality, race and religion vs across categories ($\varphi_{2}$). "Featurediff" uses a linear model applied to the difference of model features produced for the respective first tokens in $s$ and $s'$. "Bilinear" uses a bilinear model on top of these feature differences instead. "Merge" appends $s'$ to $s$ before tokenization and learns a linear model on top of the model features for this combined input. "Concat" operates similarly, but first tokenizes $s$ and $s'$ and pads both to $64$ tokens before feeding the concatenated tokens into the model. No dropout was used in the post-BERT layers for these experiments. All results averaged over 10 runs and $\pm$ indicates the difference from the upper/lower bound of a naive $95\%$ confidence interval assuming normally distributed errors. }\label{table:architectures}
\end{table}
Example generations for our different methods can be found in \cref{appendix:examples}.
\subsection{Synthetic Data}

For active learning, we freeze the underlying BERT model during the active learning selection and only apply MC-Dropout on the level of the classifier head, similar to \citep{griesshaber2020fine}, but unlike them we do not use BALD \citep{houlsby2011bayesian} and instead approximate $p(y|s,s')$ averaging the models' predicted probabilities $p_{\hat{\varphi}}(y|s,s',w)$ for $50$ sampled dropout masks $w$. We call this approach LC-UNC and experimented with various alternative selection criteria. Unlike LC-UNC, LC directly approximates $1-\max_{y}p(y|s,s')$ using a single forward pass through the $\hat{\varphi}$ with deactivated dropout. BALD is the approach from \cite{griesshaber2020fine}, while VARRA and Majority approximate $1-\max_{y}p(y|s,s')$ using MC-Dropout differently than LC-UNC: In Majority, $p(y|s,s')$ is approximated as the fraction of dropout samples $w$ for which $\hat{\varphi}=1$, while VARRA averages $1-\max_{y}p_{\hat{\varphi}}(y|s,s',w)$ over dropout samples $w$ instead of averaging $p_{\hat{\varphi}}(y|s,s',w)$ before applying the maximum operator. In addition, the table contains the "automatic relabeling" condition in which $D_{i}$ is selected from the whole of $C^e$ rather than just the previously unlabeled examples $D_{i}\subset C^e \setminus \bigcup_{j<i} D_{j}$. During training, pairs $(s,s')$ that have been queried multiple times are labelled according to the majority vote of all queries, and as $0.5$ in case of a tie.

We validate the efficacy of our active learning approach for learning the similarity function $\varphi(s,s')$ with a limited amount of noisy queries. For this, we define two synthetic similarity functions $\varphi_{i}: i\in \{1,2\}$. The first, $\varphi_{1}$ is equal to zero, whenever a pair $(s,s')$ was generated via word replacement and equal to one otherwise, as in the first experiment from the previous section. The second, $\varphi_{2}$ is equal to zero, whenever the group $j$ of $s$ that was removed and the added group $j'$ in $s'$ are within the same category of gender and sexuality, race, or religion, and equal to one otherwise. For example, a pair $(s,s')$ for which markers of "White people" in $s$ were modified to markers of "Black people" in $s'$ would have $\varphi_{2}(s,s')=0$, while $\varphi_{2}(s,s')$ would be one if the group was modified to "muslim" in $s'$ instead. We simulate the label noise introduced by annotators' disagreement by independently flipping each label with probability $p=0.3$ during training the similarity classifier $\hat{\varphi}$. For training with $3$ instead of one query per data point, we reduce the overall amount of training data from $10000$ samples in $C^e$ to $3333$ samples and reduce the probability of flipping labels to $p=0.216$, simulating a majority vote. In turn, the active learning approach selects $333$ instead of $1000$ data points for labeling in each of its ten steps in that scenario. \cref{table:AL_expand} shows that active learning noticeably outperforms randomly sampling data points for our task, that there is no clear direct benefit from employing multiple queries per pair $(s,s')\in C^e$ over obtaining labels for previously unseen pairs, an that the LC-UNC setup is usually performing as well as or better than alternative selection criteria in the one-query per data  point  setting.

\begin{table}[ht]
	\begin{center}
		\begin{tabular}{ @{}lcc@{}}
			\toprule
			Method/Dataset & $\varphi_{2}$ (Same category) &  $\varphi_{1}$ (Word replacement)   \\
			\midrule
			Random sampling, 1 query & $75.1 \pm 3.6$ & $74.8 \pm 1.8$ \\
			Random sampling, 3 queries  &  $71.6 \pm 3.9$ & $72.5 \pm 1.5$  \\
			Random sampling, 5 queries & $70.7 \pm 2.7$ & $73.4 \pm 1.8$   \\
			\midrule
			BALD 1 query & $75.9 \pm 4.0$  &  $77.9 \pm 2.1$ \\
			BALD 3 queries & $73.8 \pm 6.5$ & $78.1\pm 1.7$  \\
			BALD automatic relabeling & $76.1 \pm 4.5$& $77.6 \pm 2.6$ \\
			\midrule
			LC 1 query & $79.1 \pm 4.4$ & $78.5 \pm 1.8$  \\
			LC 3 queries & $74.6\pm 2.4$ & $79.5\pm 1.8$ \\
			LC automatic relabeling & $73.4 \pm 5.9$  & $78.2 \pm 1.3$  \\
			\midrule
			LC-UNC 1 query & $79.0 \pm 4.9$ & $79.7 \pm 1.5$   \\
			LC-UNC 3 queries & $75.8\pm 5.4$  & $78.7 \pm 2.6$  \\
			LC-UNC automatic relabeling & $76.6 \pm 3.9$ &  $76.7 \pm 1.5$ \\
			\midrule
			VARRA 1 query &  $77.3 \pm 7.4$ &  $78.9\pm 2.1$  \\
			VARRA 3 queries & $73.1 \pm 5.7$ & $79.8\pm 1.6$ \\
			VARRA automatic relabeling & $77.7\pm 2.9$  & $78.0\pm 1.3$ \\
			\midrule
			Majority 1 query & $74.9\pm3.5$  & $76.8 \pm 2.4$    \\
			Majority 3 queries & $78.7 \pm 5.2$  &  $79.6 \pm 0.9$\\
			Majority automatic relabeling & $74.4 \pm 6.2$  & $77.9 \pm 1.8$ \\
			\bottomrule
		\end{tabular}
	\end{center}
	\caption{Balanced accuracy for BERT classifier trained using a constant amount of 50k gradient steps and a constant amount of 10k queries. All results are averaged over 10 runs and $\pm$ indicates the difference from the upper/lower bound of a naive $95\%$ confidence interval assuming normally distributed errors.}\label{table:AL_expand}
\end{table}
\subsection{Human Evaluation}
\cref{table:AL sup T,table:AL sup S} show additional results on the active learning from human feedback. As above, we tested our approach using different filtering thresholds $t$ on the two test sets $T$ (\cref{table:AL sup T}) and $S$ (\cref{table:AL sup S}). In the Retrain condition, the classifier $\hat{\varphi}$ was trained for a single epoch on all labeled datapoints $\bigcup_{i<n} D_{i}$ in order to combat potential issues with catastrophic forgetting. In the Retrain + Reweigh condition, the same was done, but the Cross Entropy loss was reweighed to balance the empirical label frequencies in $\bigcup_{i<n} D_{i}$. In the From Scratch setting, we train a new classifier on $\bigcup_{i<n} D_{i}$ for $5$ epochs from scratch without first training it separately on any $D_{i}$. Again, datapoints are reweighed according to their empirical frequency in $\bigcup_{i<n} D_{i}$ in the From Scratch + Reweigh setting.
\begin{table}[ht]
\begin{center}
\begin{tabular}{ @{}lccc@{} }
\toprule
Method & ACC & TNR & TPR\\
\midrule
Baseline: Constant 0 & 78.8 & 100.0 & 0.0 \\
\midrule
AL t=0.5 & $ 79.8 \pm 0.3$  & $ 97.2 \pm 0.3 $ & $ 15.1 \pm 1.2 $ \\
AL + Relabel t=0.5 & $81.1 \pm 0.3$ & $95.5 \pm 0.7$ & $  28.6 \pm 2.2$ \\
AL + Relabel + Retrain t=0.5 & $ 79.6 \pm 0.4 $ & $ 95.3 \pm 1.4 $ & $ 21.5 \pm 3.9 $ \\
AL + Relabel + Retrain + Reweigh t=0.5 & $ 79.6 \pm 0.8 $ & $ 93.9 \pm 1.6 $ & $ 26.6 \pm 3.4 $ \\
From Scratch t=0.5 & $77.5 \pm 1.3$ & $90.8 \pm 3.3 $ & $28.1 \pm  7.1$ \\
From Scratch + Reweigh t=0.5 & $ 77.7 \pm 1.4$ & $ 91.0 \pm 2.7 $ & $ 28.3 \pm 5.0 $ \\
\midrule
AL t=0.1 & $ 80.0 \pm 0.5 $  & $ 95.2 \pm 0.7$ & $ 23.7 \pm 3.5  $\\
AL + Relabel t=0.1 &$ 80.7 \pm 0.6 $ & $ 93.0 \pm 0.9 $ & $ 35.0 \pm 1.3 $ \\
AL + Relabel + Retrain t=0.1 & $62.1 \pm 5.6 $ & $61.5\pm 8.9 $ & $ 64.0 \pm 7.0$ \\
AL + Relabeling + Retrain + Reweigh t=0.1 & $52.8 \pm 6.2 $ & $ 46.8 \pm 7.7$ & $75.0 \pm 4.6 $ \\
From Scratch t=0.1 & $53.4 \pm 7.9$ & $48.6 \pm 14.3$ & $71.1 \pm 9.2$ \\
From Scratch + Reweighed t=0.1 & $54.8\pm 6.7$ & $51.2 \pm 10.5$ & $67.9 \pm 9.1 $ \\
\midrule
AL t=0.01 &$78.7 \pm 1.1 $ & $87.5 \pm 2.1$ & $45,7 \pm 1.8$  \\
AL + Relabel t=0.01 & $78.3 \pm 0.7$ &$86.8\pm 1.5$ & $46.6\pm 2.5$ \\
AL + Relabel + Retrain t=0.01 & $21.2 \pm 0.1$ & $0.0\pm 0.0$ & $100 \pm 0.0 $\\
AL + Relabel + Retrain + Reweigh t=0.01 & $21.1 \pm 0.0 $ & $ 0.0 \pm 0.0 $ & $ 100 \pm 0.0 $ \\
From Scratch t=0.01 &$21.7 \pm 0.5$ & $0.0 \pm 0.0$ & $99.5 \pm 0.6$ \\
From Scratch + Reweigh t=0.01 & $  21.8 \pm 1.5 $ & $ 1.5 \pm 3.6 $ & $ 98.3 \pm 1.7$ \\
\bottomrule
\end{tabular}
\end{center}
\caption{Results for active learning to predict human fairness judgments, on test data $T$. Active learning classifiers are retrained $10$ times on the last batch $D_{6}$. Results are averaged and $\pm$ indicates the difference from the upper/lower bound of a naive $95\%$ confidence interval assuming normally distributed errors.  }\label{table:AL sup T}
\end{table}

\begin{table}[ht]
\begin{center}
\begin{tabular}{ @{}lccc@{} }
\toprule
Method & ACC & TNR & TPR\\
\midrule
Baseline: Constant 0 & $96.1$ & $100.0$ & $0.0$ \\
\midrule
AL t=0.5 & $ 93.8 \pm 0.5 $ & $ 97.0 \pm 0.6 $ & $ 14.6 \pm 2.2$ \\
AL + Relabel t=0.5 & $92.1 \pm 0.6$ & $95.1 \pm 0.7$ & $18.9 \pm 2.7$ \\
AL + Relabel + Retrain t=0.5 & $ 90.7 \pm 1.7 $ & $93.8 \pm 1.9$ & $12.8 \pm 4.0 $ \\
AL + Relabel + Retrain + Reweigh t=0.5 & $ 89.0 \pm 1.3 $ & $ 92.0 \pm 1.4 $ & $16.4 \pm 3.4$  \\
From Scratch t=0.5 & $89.2 \pm 2.6$ & $91.8\pm 2.5$ & $25.7\pm 5.5$ \\
From Scratch + Reweigh t=0.5 & $ 89.2 \pm 2.5 $ & $ 91.8 \pm 2.7 $ & $25.7\pm 4.4 $ \\
\midrule
AL t=0.1 & $ 90.4 \pm 1.3 $ & $ 93.3 \pm 1.3 $ & $ 21.0 \pm 2.3 $ \\
AL + Relabel t=0.1 &$ 89.6 \pm 0.8 $ & $ 92.2 \pm 0.8 $ & $ 24.6 \pm 1.4 $\\
AL + Relabel + Retrain t=0.1 & $60.0 \pm 8.1$ & $59.5 \pm 8.8$ & $ 72.8 \pm 11.9 $\\
AL + Relabel + Retrain + Reweigh t=0.1 & $46.7  \pm 7.4 $ & $ 45.2 \pm 8.0$ & $ 83.9 \pm 7.6$ \\
From Scratch t=0.1 & $50.6 \pm 10.4$ & $49.8 \pm 11.2$ & $69.6 \pm 9.3$ \\
From Scratch + Reweigh t=0.1 & $55.0 \pm 9.4 $ & $ 54.5 \pm 10.0$ & $ 66.7 \pm 6.6$ \\
\midrule
AL t=0.01 & $80.6 \pm 2.3$ &$82.3 \pm 2.7$ & $38.2 \pm 6.8$\\
AL + Relabel t=0.01 & $80.2 \pm 1.3$& $85.5\pm 1.4$ & $30.0\pm 2.7$ \\
AL + Relabel + Retrain t=0.01 & $3.9 \pm 0.0 $ & $ 0.0 \pm 0.0$ & $ 100.0 \pm 0.0 $ \\
AL + Relabel + Retrain + Reweigh t=0.01 & $ 3.9 \pm 0.0 $ & $ 0.0 \pm 0.0 $ & $ 100.0 \pm 0.0 $ \\
From Scratch t=0.01 & $4.6 \pm 0.9$ & $0.0 \pm 0.1$ & $99.6 \pm 0.4$ \\
From Scratch + Reweigh t=0.01 & $ 5.4 \pm 3.9 $ & $ 1.6 \pm 3.2 $ & $ 50.8 \pm 1.6$ \\
\bottomrule
\end{tabular}
\end{center}
\caption{Results for active learning to predict human fairness judgments, using the separate test data $S$.  Active learning classifiers are retrained $10$ times on the last batch $D_{6}$. Results are averaged and $\pm$ indicates the difference from the upper/lower bound of a naive $95\%$ confidence interval assuming normally distributed errors. }\label{table:AL sup S}
\end{table}

\clearpage
\section{Further Details on Training Downstream classifiers}\label{appendix:filtering_appendix}
The downstream classifier $f$ consists of a pretrained roberta-base model followed by a linear layer with $768$ neurons applied to the output embedding of the first token, a Tanh layer, another linear layer mapping to a single dimension, and a Sigmoid layer. We train $f$ using binary Cross Entropy reweighed to balance the empirical label frequencies in $D$ for $3$ epochs using a batch size of $32$ and the Adam optimizer with a learning rate of $0.00001$.

\cref{table:robtrans_long} extends \cref{table:robtrans} and shows that censoring words yields very strong constraint adherence for the respective word list \footnote{Artifacts like word replacement lists that contain both a word $s$ and substrings of $s$ keep this below 100\%}. However, we find it to generalize worse than CLP trained with the same word list, both to our style transfer pairs, and even to the respective other word list. Similarly, we find that training with CLP on $C^e$ or our style transfer pairs $C_{3}$ does not just yield significantly improved constraint adherence on $C_{3}$, but also generalizes better to $C_{1}$ and $C_{2}$ than the respective other of the two word replacement constraint sets without losing much downstream accuracy. Lastly, the table also shows that the better generalization from style transfer to word replacement persists for large values of $\lambda$ in CLP and that these values can provide strong improvements in terms of fairness, albeit at a larger cost in terms of balanced accuracy.

\begin{table}[ht]
	\begin{center}
		\begin{tabular}{@{}lccccc@{}}
			\toprule
			Training/Evaluation &  BA & WR$_{50}$ ($C_{1}$) & WR ($C_{2}$)  & ST ($C_{3}$) & Full $C^e$ \\
			\midrule
			Baseline & $ 88.4 \pm 0.1  $& $ 78.4 \pm 1.4 $ & $ 81.3 \pm 1.5 $  & $ 76.7 \pm 1.8 $  &  $ 78.5 \pm 1.5$  \\
			\midrule
			Censoring WR$_{50}$& $ 87.0 \pm 0.3  $& $ 99.8 \pm 0.0$ & $ 88.4 \pm 1.2 $ & $84.7 \pm 1.1$  &  $ 85.9 \pm 1.2 $   \\
			Censoring WR & $ 86.1 \pm 0.4  $& $ 91.4 \pm 1.2$ & $  99.3 \pm 0.2 $  & $89.0 \pm 1.5$  &  $ 92.8 \pm 1.0 $  \\
			Censoring Both WR & $ 86.2 \pm 0.3  $& $99.7 \pm 0.2 $ & $ 99.1 \pm 0.1 $ & $89.3 \pm 0.4$  &  $ 92.8 \pm 0.3 $   \\
			\midrule
			CLP($\lambda=$$5$) WR$_{50}$($C_{1}$)  & $ 87.0 \pm 0.3 $& $ 98.3 \pm 0.1 $ & $ 89.1 \pm 1.9 $& $86.3 \pm 1.9$  &  $ 87.3 \pm 1.8 $  \\
			CLP($\lambda=$$5$) WR ($C_{2}$) & $ 87.2 \pm 0.1 $& $ 93.1\pm 1.2$ & $98.2\pm 0.4$ & $90.5 \pm 1.7 $ &  $92.9\pm 1.2$    \\
			CLP($\lambda=$$5$) ST ($C_{3}$) & $ 85.9 \pm 0.1 $& $ 95.3\pm 0.4$ & $97.1 \pm 0.3$ & $95.4 \pm 0.4$ &  $95.5\pm 0.3$    \\
			CLP($\lambda=$$5$)  Full $C^e$  & $ 85.0 \pm 3.4 $&  $95.5 \pm 0.9 $ & $ 97.8 \pm 0.6 $ & $94.9 \pm 0.9$ & $ 95.7 \pm 0.8 $  \\
			\midrule
			CLP($\lambda=$$125$) WR$_{50}$($C_{1}$) & $ 82.5 \pm 1.3 $& $ 98.3 \pm 0.6 $ & $ 94.3 \pm 0.8  $ & $90.9 \pm 1.1$ &  $ 92.1 \pm 0.9 $  \\
			CLP($\lambda=$$125$) WR ($C_{2}$)& $ 81.8 \pm 1.5 $& $ 95.9 \pm 2.2 $ & $98.6\pm 0.5$ & $92.5\pm 2.2$  & $94.7 \pm 1.5$    \\
			CLP($\lambda=$$125$) ST ($C_{3}$) & $ 80.3 \pm 2.8 $& $ 97.6 \pm 0.8 $ & $98.4\pm 0.6$  &  $97.2\pm 0.9$ &  $97.2 \pm 0.9$    \\
			CLP($\lambda=$$125$) Full $C^e$ & $ 79.3 \pm 6.1 $&  $ 97.8 \pm 1.3 $ & $ 98.6 \pm 0.9   $ & $97.1 \pm 1.6$ & $ 97.4 \pm 1.4 $  \\
			\bottomrule
		\end{tabular}
	\end{center}
	\caption{Balanced accuracy and individual fairness (proportion of similar pairs $(s,s')\in C_{i}$ for which $f(s)=f(s')$) for a Roberta-based classifier $f$ trained with CLP using different constraint sets for training. Results reported with $\pm$ are averaged over 5 runs and $\pm$ indicates the difference from the upper/lower bound of a naive $95\%$ confidence interval assuming normally distributed errors. }\label{table:robtrans_long}
\end{table}

\subsection{Experiments with filtering on synthetic data}
The filtering process for CLP is implemented as follows: for each batch $B$ of labeled training examples $(s,y(s))$ used to train a downstream classifier $f$, we evaluate $p_{\hat{\varphi}}(s,s')$ for all $(s,s')\in C^e$ with $s\in B$. Then, for every $s\in B$ we randomly select a pair $(s,s')$ among the pairs with $p_{\hat{\varphi}}(s,s')>t$ for a filtering threshold $t$ to use in the CLP regularizer $\lambda ||l(s)-l(s')||_{2}$ with $l$ representing the logits of the downstream classifier $f$, using $(s,s)$ if no such pair exists.
To allow for more precise control over the statistical properties of $\hat{\varphi}$, we constructed additional $\hat{\varphi}$ using a look-up table using $\varphi_{i}$ and flip the labels of randomly selected pairs $(s,s')$ with either $\varphi_{i}(s,s')=1$ or $\varphi_{i}(s,s')=0$ in order to achieve specific True positive rates (TPR) and true negative rates (TNR).
\cref{table:RobWR_extend} shows that there are consistent benefits from filtering for the synthetic similarity function $\varphi_{1}$ from \cref{appendix:active_appendix} across different values of $\lambda$, even when an imperfect $\hat{\varphi}$ with a TPR and TNR of $75\%$ is used.

\begin{table}[ht]
	\begin{center}
		\begin{tabular}{ @{}lcc@{}}
			\toprule
			Method & Balanced Accuracy & Fairness \\
			\midrule
			Baseline & $\textit{88.2}\pm 0.4$ &  $\textit{82.0} \pm 2.2$ \\
			Full $C^e$ $\lambda = 5.0$ & $85.6 \pm 0.4$ & $98.2 \pm 0.2$   \\
			Full $C^e$ $\lambda = 125.0$ & $73.9 \pm 16.8$ & $98.6 \pm 1.0$  \\
			Filtering with 75$\%$ TNR/TPR, $\lambda = 5.0$  & $86.3 \pm 0.6$ &  $97.9 \pm 0.3$  \\
			Filtering with 75$\%$ TNR/TPR, $\lambda = 125.0$  & $77.2 \pm 6.1$ &  $99.1 \pm 0.3$  \\
			Perfect filtering $\lambda = 5.0$  & $\textbf{87.5} \pm 0.1$ &  $98.2 \pm 0.2$  \\
			Perfect filtering $\lambda = 125.0$ & $86.1 \pm 0.4$  & $\textbf{99.3} \pm 0.1$  \\
			\bottomrule
		\end{tabular}
	\end{center}
	\caption{Balanced accuracy and individual fairness (proportion of similar pairs $(s,s')$ according to $\varphi_{1}$ for which $f(s)=f(s')$) CLP training after filtering $C^e$ using approximations of $\varphi_{1}$ with varying error profiles. All results are averaged over 5 runs and $\pm$ indicates the difference from the upper/lower bound of a naive $95\%$ confidence interval assuming normally distributed errors. }\label{table:RobWR_extend}
\end{table}
\ref{table:RobAxis} shows that unlike for $\varphi_{1}$ (\cref{table:RobWR_extend}), there is little gain from filtering constraints for $\varphi_{2}$, most likely because some of the constraint candidates generated by GPT-3 and our style transfer approach are difficult to enforce while maintaining high level of accuracy. While all of these constraints are inactive for $\varphi_{1}$ and are therefore not enforced with sufficiently accurate filtering, many of them remain active with $\varphi_{2}$ such that filtering yields no clear benefits.
\begin{table}[ht]
	\begin{center}
		\begin{tabular}{ @{}lcc@{}}
			\toprule
			Method & Balanced Accuracy & Fairness \\
			\midrule
			Baseline & $\textit{87.9}\pm 1.2$ & $\textit{76.5} \pm 1.5$\\
			Full $C^e$ $\lambda = 5.0$ & $85.6 \pm 0.4$ & $96.6 \pm 0.4$   \\
			Full $C^e$ $\lambda = 125.0$ & $78.9 \pm 2.7$ & $\textbf{97.5} \pm 1.2$  \\
			Perfect filtering $\lambda = 5.0$  & $\textbf{86.6} \pm 0.3$ &  $95.7 \pm 0.6$  \\
			Perfect filtering $\lambda = 125.0$ & $80.7 \pm 2.2$  & $97.3 \pm 0.6$  \\
			\bottomrule
		\end{tabular}
	\end{center}
	\caption{Balanced accuracy and indvidual fairness (proportion of similar pairs $(s,s')$ according to $\varphi_{2}$ for which $f(s)=f(s')$) for CLP training after filtering $C^e$ using approximations of $\varphi_{2}$ with varying error profiles. All results are averaged over 5 runs and $\pm$ indicates the difference from the upper/lower bound of a naive $95\%$ confidence interval assuming normally distributed errors. }\label{table:RobAxis}
\end{table}
\subsection{Experiments with filtering on human fairness judgments}

\cref{table:endresults_full} is an extended version of \cref{table:endresults} including additional experiments with a larger regularization parameter $\lambda=125$. Again, there is no visible benefit from filtering. Counterintuitively, more filtering appears to correspond to less accuracy but slightly more fairness, but this might be by chance, given the significantly larger error bars for $\lambda=125$.

\begin{table}[ht]
	\begin{center}
		\begin{tabular}{ @{}lcccc@{} }
			\toprule
			Method & BA & NR & Fairness ($T$) & Fairness ($S$) \\
			\midrule
			Baseline & $ 88.2 \pm 0.4 $ & $0.0$ & $ 82.1 \pm 2.1 $ & $ 84.7 \pm 1.3 $  \\
			\midrule
			WR (Garg)  $\lambda=5$& $  87.1 \pm 2.0  $ & $100$ & $ 92.8 \pm 0.9  $ & $ 95.2 \pm 0.8 $ \\
			WR  $\lambda=5$& $ 87.2 \pm 0.2 $ & $100$ &$ 95.8 \pm 0.9$ & $ 95.8 \pm 1.2 $  \\
			Full constraint set $C^e$ $\lambda=5$&  $85.9 \pm 0.3$ & $100$ &$96.5 \pm 1.4$ & $97.0 \pm 1.5$  \\
			Filtering with threshold 0.5 $\lambda = 5$ & $85.9 \pm 0.5$ & $88.5 \pm 1.0 $& $97.4\pm 1.1$ & $97.1 \pm 1.1$\\
			Filtering with threshold 0.1 $\lambda = 5$ & $86.1 \pm 0.1$ & $84.6 \pm 1.4$ &$97.2 \pm 0.6$ & $96.6 \pm 0.6$ \\
			Filtering with threshold 0.01 $\lambda = 5$ & $85.9 \pm 0.2$ &$76.9 \pm 2.0$  &$97.1 \pm 1.0$ & $96.9 \pm 1.1$ \\
			\midrule
			WR (Garg) $\lambda=125$& $  81.6 \pm 0.6  $ &$100$& $ 95.6 \pm 1.7  $ & $ 96.8 \pm 0.2 $ \\
			WR $\lambda=125$& $  81.2 \pm 2.7  $ &$100$& $ 97.4 \pm 2.5  $ & $ 97.5 \pm 0.1 $ \\
			Full constraint set $C^e$  $\lambda=125$&  $ 81.8 \pm 2.1$ &$100$& $98.0 \pm 0.6$ & $97.6 \pm 1.1$  \\
			Filtering with threshold 0.5 $\lambda = 125$ & $81.3 \pm 1.5$ & $88.5 \pm 1.0 $ &$98.1\pm 0.9$ & $98.0 \pm $ 0.7\\
			Filtering with threshold 0.1 $\lambda = 125$ & $80.7 \pm 1.9$ &$84.6 \pm 1.4$&  $97.8 \pm 1.3$ & $98.2 \pm 0.6$ \\
			Filtering with threshold 0.01 $\lambda = 125$ & $80.7 \pm 2.2$ & $76.9 \pm 2.0$ & $98.6 \pm  0.5$ & $97.8 \pm  0.9$\\
			\toprule
		\end{tabular}
	\end{center}
	\caption{Balanced accuracy, percent of unfiltered pairs, and individual fairness on two different test sets for differently trained classifiers $f$. Fairness is measured as the fraction of pairs $(s,s')$ with $\varphi(s,s')=0$ according to the majority of annotators, for which $f(s)=f(s')$. Active learning classifiers are retrained $5$ times on the last batch $D_{6}$ and results are averaged over 5 training runs. $\pm$ indicates the difference from the upper/lower bound of a naive $95\%$ confidence interval assuming normally distributed errors.}\label{table:endresults_full}
\end{table}

\clearpage
\section{Word Lists And Example Generations}\label{appendix:examples}
\cref{table:sttable,table:gpttable,table:edittable} show 5 randomly selected example pairs $(s,s')$ produced by our style transfer approach and GPT-3 in zero-shot and edit mode. \cref{table:gencomp} shows 5 generated comments for each method using the same source comment and target demographic attribute.

 \textcolor{red}{Warning: Some of the example texts contain offensive language.}
\begin{table}
	\begin{center}
		\begin{tabular}{ |p{0.2\linewidth}|p{0.35\linewidth}|p{0.35\linewidth}| }
			\hline
			Demographic Group & Descriptors & Nouns \\
			\hline
			Male  & male,manly,masculine & man, men, grandfather, bro, guy, boy, father, dad,
			son, husbands, husband, grandpa, brother \\
			\hline
			Female  & female, pregnant, feminine, femme, womanly & woman, women, grandmother, lady, ladies, girl, mother, mom,
			daughter, wives, wife, grandma, sister \\
			\hline
			Transgender & transsexual, FTM, F2M, MTF, trans,  M2F, transgender, trans female, trans feminine,
			trans femme,trans fem, trans male,trans masculine,trans masc & \\
			\hline
			Heterosexual & straight, heterosexual, hetero & \\
			\hline
			Homosexual & gay, homosexual, lesbian & \\
			\hline
			Christian & Amish, Catholic, Christian,  Evangelical, Lutheran, Mormon, Protestant,
			Unitarian & \\
			\hline
			Jewish & Conservative Jewish, Orthodox Jewish,  Reform Jewish, Hasidic, Jewish & \\
			\hline
			Muslim & Muslim, Shia, Sunni & \\
			\hline
			Hindu & Hindu & \\
			\hline
			Buddhist & Buddhist & \\
			\hline
			Atheist & atheist, irreligious & \\
			\hline
			Black & African-American, African, Black, BIPOC & \\
			\hline
			White & European-American, European, Caucasian, White & \\
			\hline
			Asian & Chinese-American, American-Born Chinese,  Chinese, Indian-American, Indian,
			Korean-American,
			Korean, Vietnamese-American,  Vietnamese, AAPI & \\
			\hline
			Latino & Cuban-American, Cuban, Dominican-American,  Dominican, Salvadoran-American, Salvadoran, Guatemalan-American, Guatemalan, Mexican-American, Mexican, Filipina-American, Filipina, Filipino-American, Filipino, Hispanic, Latinx, Latine, Latino, Latina, Latin American & \\
			\hline
		\end{tabular}
	\end{center}
	\caption{Lists of words used for word replacementm, based on \citet{smith2022m}}\label{table:WR}
\end{table}

\begin{table}[ht]
	\begin{center}
		\begin{tabular}{ |p{0.45\linewidth}|p{0.45\linewidth}| }
			\hline
			Original Comment & Modified Comment \\
			\hline
			'Hahaha, really?  Thousands of years of war and it was mostly white men, huh?  Read some history about non-European cultures to pierce your idyllic bubble.' & 'Hahaha, really?  Thousands of years of war and it was mostly white European Christians, huh?  Read some history about non-European cultures to pierce your idyllic bubble'\\
			\hline
			'I would really like to see your source for the Japanese fear of the Russians as being any kind of factor whatsoever. For most of the Japanese people, the Russians were not even on the radar screen!!' & 'I would really like to see your source for the Buddhist fear of the west as being any kind of factor whatsoever. For most of the western people, the west and religion were not even on the radar screen!!' \\
			\hline
			'Yes: at least implicitly it endorses the old canard that the United States is a "Christian nation."' &
			'Yes: at least implicitly it endorses the old canard that the United States is a "white racist nation."'\\
			\hline
			'Please be careful what you label fake news. News of shootings (Greenville?), migrants in Sweden raising crime rates, Hispanics being rapists and murderers, illegal voting, etc those are examples of fake news....feet on couch? not so much....' &
			'Please be careful what you label fake news. Hundreds of shootings (Green Mountain?), migrants in Sweden raising crime rates, homosexuals being raped and murderers, illegal voting, etc those are examples of fake news....feet on couch? not so much.'
			\\
			\hline
			'And often they are not so silent.  Gay friends have told me about various such clergy who are well known -- only no one speaks out about it.' & 'And often they are not so silent.  Gay friends have told me about various such men who are well known -- only no one speaks out about it.' \\
			\hline
		\end{tabular}
	\end{center}
	\caption{Examples of pairs produced by our style transfer approach}\label{table:sttable}
\end{table}

\begin{table}[ht]
	\begin{center}
		\begin{tabular}{ |p{0.45\linewidth}|p{0.45\linewidth}| }
			\hline
			Original Comment & Modified Comment \\
			\hline
			'OK, so now I\'m confused.  Is it or is it not OK to mock a reporter\'s physical handicap?  Call a Latina beauty queen "Miss Housekeeping"?  Post a mock video showing yourself clocking a woman you don\'t like with a hard-hit golf ball?'
			&
			'  Is it or is it not OK to mock a reporter\'s physical handicap? Call a Muslim beauty queen "Miss Housekeeping"? Post a mock video showing yourself clocking a woman you don\'t like with a hard-hit golf ball?'\\
			\hline
			"You caught me. I confess. I've been lying this whole time. I'm not actually an atheist, of course I believe in a god. I see him in the mirror every day !"&
			"  You caught me. I confess. I've been lying this whole time. I'm not actually a homosexual, of course I believe in a god. I see him in the mirror every day !"\\
			\hline
			"Denying the existence of racism in America,(or most anywhere), is just ignoring reality.    The formal dicrimination has largely been curbed(with exception of using race to design Congressional districts), but those wonderful 'family values' of the white South still carry the mythology of inherent white superiority" &
			".  Denying the existence of racism in America,(or most anywhere), is just ignoring reality.    The formal dicrimination has largely been curbed(with exception of using race to design Congressional districts), but those wonderful 'family values' of the Hindu South still carry the mythology of inherent Hindu" \\
			\hline
			'Cite to "the Republicans, Trump being a Republican, who are against gays" please.  New one to me. (Other than Trump being registered Republican of course, lol!)' &
			'  Christian people are against gays.'\\
			\hline
			"You'd better watch out if you don't claim Christianity unto yourself like they do. It is true that people harbor resentment over people whom they view as different.  It's the only reason I can think of why they hate gay people." &
			"  You'd better watch out if you don't claim Christianity unto yourself like they do. It is true that people harbor resentment over people whom they view as different. It's the only reason I can think of why they hate Asians."
			\\
			\hline
		\end{tabular}
	\end{center}
	\caption{Examples of pairs produced by GPT-3 (zero-shot)}\label{table:gpttable}
\end{table}

\begin{table}[ht]
	\begin{center}
		\begin{tabular}{ |p{0.45\linewidth}|p{0.45\linewidth}| }
			\hline
			Original Comment & Modified Comment \\
			\hline  'So do they cover their faces?  This is about being able to identify women for practical purposes.  Makes sense to moi.' &
			'So do they cover their faces?  This is about being able to identify White people for practical purposes.  Makes sense to moi. '
			\\
			\hline 'What about LGBTQI+M(xn) W(omxn)Y(oxth)C(hxldrxn)O(thxrs)' &
			'Christianity is the best religion. Christian people are the best people. Jesus Christ was a man who lived a perfect life and died for our sins. We should all be following Jesus because he is the author and finisher of our faith. '  \\
			\hline
			'Bust Judaism, Christianity and Islam are all Abrahamic religions tracing their origins to the patriarch Abraham.  They are all related.' &
			'Bust Hispanic people are all the same. They are all related. '
			\\
			\hline
			"AO96 doesn't allow men into womens rest rooms or locker rooms. You do realize that trans people have been using the bathroom with you forever right? Have you ever been injured by any of these people in all the years you have been using public facilities?"&
			"AO96 doesn't allow white people into black rest rooms or locker rooms. You do realize that trans people have been using the bathroom with you forever right? Have you ever been injured by any of these people in all the years you have been using public facilities? "\\
			\hline
			'so.. I . have seen dozens of women in niqab in Toronto this week alone Almost every day I see groups of them. The problem with the niqab policy is that is was for such a silly narrow issue. The niaqb should have been banned from public buildings and shopping malls.' &
			'so.. I . have seen dozens of black people in the niqab in Toronto this week alone Almost every day I see groups of them. The problem with the niqab policy is that is was for such a silly narrow issue. The niaqb should have been banned from public buildings and shopping malls. ' \\
			\hline
		\end{tabular}
	\end{center}
	\caption{Examples of pairs produced by GPT-3 (edit mode)}\label{table:edittable}
\end{table}

\begin{table}[ht]
	\begin{center}
		\begin{tabular}{ |p{0.22\linewidth}|p{0.22\linewidth}| p{0.22\linewidth}| p{0.22\linewidth}| }
			\hline
			Original Comment & Modified Comment (Style Transfer) & Modified Comment (GPT-3 zero shot & Modified Comment (GPT-3 edit mode)   \\
						\hline
			I agree with those many trans-gendered. & I agree with those many pre-gospel Catholics. & I agree with those many Christians. & I agree with those many Christian people.  \\
					   \hline
			Everyone in plain black pajamas might do the trick. & Everyone in plain white pajamas might do the trick &  All white people in plain white pajamas might do the trick. & Everyone in plain white pajamas might do the trick. \\
		   \hline
		   You do know Mexicans are American ...right? & You do know gays are gay...right? &  You do know homosexuals are American ...right? & You do know Homosexuals are people too ...right? . \\
		   \hline
		   I'm an atheist and I don't care about obnoxious atheists. Either way. &  I'm an Indian and I don't care about obnoxious Muslims. Either way. &    I don't care about obnoxious Hindus. & I'm a Hindu and I don't care about obnoxious Hindus. Either way.  \\
		   \hline
		   Another condescending ignorant Hindu nationalist ignorant of the genocidal, racist and murderous past of its culture. & Another condescending ignorant white nationalist ignorant of the genocidal, racist and murderous past of its culture &  Another condescending ignorant white nationalist ignorant of the genocidal, racist and murderous past of its culture. & A White person. \\ \hline
		\end{tabular}
	\end{center}
	\caption{Examples of pairs produced by different methods for the same source comment and target demographic attribute.}\label{table:gencomp}
\end{table}

\clearpage
\section{Group Fairness}
\cref{table:gfresults} provides results on equality of odds for a
subset of the models trained for the experiment presented in
\cref{table:endresults_full}. Specifically, for every group $j$
that was considered for generation, we calculated the TPR and TNR on
the test set restricted to comments $s$ mentioning that group
(label $y_j(s)$=1). We then calculated the absolute pairwise
differences between the TPRs/TNRs for every pair of groups (in percent),
and present the mean and maximum over these for each model.
We find that CLP-training improves the TNR gap at the cost of a worse
TPR gap, except for pairs generated by the Word Replacement list of
\citet{garg2019counterfactual}.  This effect is most extreme when using
our full constraint set $C^e$ and more mellow for Word Replacement based
on the word list of \citet{smith2022m} or for the filtered $\hat{C}^r$.
The improved TNR gap at the cost of a worse TPR gap is directionally
consistent with the results reported by \citet{garg2019counterfactual} for
CLP training using their word list, with groups defined by the presence
of specific identity terms from the word list rather than the labels
$y_j(s)$. We echo \citet{garg2019counterfactual}'s recommendation for
practicioners to select a method depending on their relative prioritization
between improving individual fairness and equitable True Negative Rates,
compared to equitable True Positive Rates.
\begin{table}[ht]
	\begin{adjustbox}{width=\textwidth,center}
		\begin{tabular}{ @{}lcccc@{} }
			\toprule
			Method & TPR Gap (Mean) & TNR Gap (Mean)  & TPR Gap (Max) & TNR Gap (Max) \\
			\midrule
			Baseline & $ 5.8 $ & $20.9$ & $ 13.6 $ & $ 54.3 $  \\
			\midrule
			WR (Garg)  $\lambda=5$& $  5.4  $ & $20.3$ & $ 13.6  $ & $ 51.8$ \\
			WR  $\lambda=5$& $ 11.2 $ & $10.8$ &$ 34.7$ & $ 30.1 $  \\
			Full constraint set $C^e$ $\lambda=5$&  $15.7$ & $6.1$ &$56.3$ & $22.0$  \\
			Filtering with $t=0.5$ $\lambda = 5$ & $11.1$ & $10.0$& $34.5$ & $29.3$\\
			\bottomrule
		\end{tabular}
	\end{adjustbox}
	\caption{Mean and Maximum of absolute pairwise TPR and TNR gaps (Absolute differences between TPR, in percent) between comments mentioning different demographic groups for differently trained classifiers $f$. }\label{table:gfresults}
\end{table}